%% file: paper.tex
\documentclass[10pt,journal,compsoc]{IEEEtran}
\usepackage{tikz}
\tikzset{node distance=1.5cm, auto}
\usepackage{import}

\ifCLASSOPTIONcompsoc
  \usepackage[nocompress]{cite}
\else
  \usepackage{cite}
\fi
\ifCLASSINFOpdf
\else
\fi
\usepackage{amsmath,amssymb,amsthm,amsrefs}
\usepackage{color}
\usepackage{tensor}
\newtheorem{theorem}{Theorem}
\newtheorem{lemma}{Lemma}

\newtheorem{remark}{Remark}

\usepackage[ruled]{algorithm2e}

\usepackage{mathtools}

\newcommand{\R}{\mathbb R}
\def\M{{\mathcal M}}
\def\SO{{\mathrm{SO}}}
\def\O{{\mathrm{O}}}
\def\GL{{\mathrm{GL}}}
\DeclareMathOperator{\diag}{\mathrm{diag}}
\DeclareMathOperator{\supp}{\mathrm{supp}}
\DeclareMathOperator{\argmin}{\mathrm{argmin}}
\DeclareMathOperator{\argmax}{\mathrm{argmax}}
\DeclareMathOperator{\E}{\mathrm{E}}

\DeclarePairedDelimiterX{\norm}[1]{\lVert}{\rVert}{#1}
\DeclarePairedDelimiterX{\set}[1]{\{}{\}}{#1}
\DeclarePairedDelimiterX{\ip}[1]{<}{>}{#1}
\DeclareMathOperator{\Exp}{Exp}
\DeclareMathOperator{\Log}{Log}
\newcommand{\Cc}{{\mathcal C}}
\newcommand{\din}{{d_{\mathrm{in}}}}
\newcommand{\dout}{{d_{\mathrm{out}}}}

\begin{document}
\title{Horizontal Flows and Manifold Stochastics in Geometric Deep Learning}

\author{Stefan~Sommer %~\IEEEmembership{Member,~IEEE,}
and Alex~Bronstein
%        John~Doe,~\IEEEmembership{Fellow,~OSA,}
%        and~Jane~Doe,~\IEEEmembership{Life~Fellow,~IEEE}% <-this % stops a space
\IEEEcompsocitemizethanks{\IEEEcompsocthanksitem S. Sommer is with the Department
of Computer Science, University of Copenhagen, Copenhagen,
Denmark.\protect\\
E-mail: sommer@di.ku.dk
\IEEEcompsocthanksitem Alex Bronstein is with the Department of Computer Science,
Technion – Israel Institute of Technology, Haifa, Israel.
}% <-this % stops an unwanted space
%\thanks{Manuscript received ; revised .}
}

\IEEEtitleabstractindextext{%
\begin{abstract}
We introduce two constructions in geometric deep learning for 1) transporting orientation-dependent convolutional filters over a manifold in a continuous way and thereby defining a convolution operator that naturally incorporates the rotational effect of holonomy; and 2) allowing efficient evaluation of manifold convolution layers by sampling manifold valued random variables that center around a weighted diffusion mean. Both methods are inspired by stochastics on manifolds and geometric statistics, and provide examples of how stochastic methods -- here horizontal frame bundle flows and non-linear bridge sampling schemes, can be used in geometric deep learning. We outline the theoretical foundation of the two methods, discuss their relation to Euclidean deep networks and existing methodology in geometric deep learning, and establish important properties of the proposed constructions.
\end{abstract}

}

\maketitle

\IEEEdisplaynontitleabstractindextext

% \fi
\IEEEpeerreviewmaketitle

\IEEEraisesectionheading{\section{Introduction}\label{sec:introduction}}


Geometric deep learning \cite{bronstein_geometric_2017} concerns the generalization of deep neural network methodology to geometric domains. Focusing on convolutional networks, the complexity in such a generalization appears both in the case where the \emph{domain} of the input signal is non-Euclidean, e.g. a manifold or a graph, and in the case where the \emph{target} of the neural network has geometric structure. A major difficulty in the first case is the fact that translation invariance of the Euclidean convolution operator does not have a direct manifold equivalent: The topology of the geometric space often prevents a continuous transport of the orientation of a filter, and the holonomy of a curved manifold prevents a notion of parallel translation that is independent of the path between points. In particular, parallel translation along minimizing geodesics is not continuous when moving points across the cut locus. In the second case, the weighted Fr\'echet mean has been proposed as a generalization of the Euclidean convolution to produce manifold valued output. Here, a practical concern is the computational complexity involved in computing the Fr\'echet mean on general manifolds.

\subsection{Motivation and contributions}
In this paper, motivated by the difficulty in transporting orientations in a continuous way and by the computational complexity of the optimization steps needed for computing the Fr\'echet mean, we derive two constructions that seek to provide new perspectives on these challenges. In this lies an investigation of the effect of curvature when combining convolution layers. The paper thereby presents the following contributions:

First, we build on the idea of orientation functions \cite{poulenard_multi-directional_2018} and the use of gauges \cite{cohen_gauge_2019} to show how curvature affects orientations as they are transported backwards through the layers of a multilayer network. The result is a time-discrete horizontal flow in the bundle $OM$ of orthonormal frames of the tangent bundle $TM$. In relation to gauge equivariant networks \cite{cohen_gauge_2019}, the focus here is on the coupling between \emph{transport} of directions and curvature as opposed to equivariance of the convolution operation to gauge transformations. 

Secondly, we use the frame bundle and a connection to show how a notion of global parallel transport that circumvents the complexities of nontrivial topology and curvature can be constructed. The idea builds on the Eells-Elworthy-Malliavin construction of Brownian motion \cite{elworthy_geometric_1988} that uses horizontal frame bundle flows to construct the Brownian motion on nonlinear manifolds. In the frame bundle, the process results in a distribution of orientations over each point of $M$, and we use such distributions to construct a convolution operator that transports filters globally over the manifold. The construction is geometrically natural in avoiding linearization to a single tangent space. We build on this idea to construct multilayer convolutions using the anti-development of the Brownian motion, resulting in a construction that is both equivariant to frame (gauge) changes, and has a smooth integrand when $M$ is analytic.

Thirdly, we combine convolution using the weighted Fr\'echet mean \cites{pennec_riemannian_2006,chakraborty_manifoldnet_2020,chakraborty_deep_2019} with the notion of diffusion means on manifolds: center points of a Brownian motion that maximizes the likelihood of a set of manifold valued data. While the weighted Fr\'echet mean (wFM) can be efficiently computed on manifolds when closed form expressions for geodesics is available, it is computationally more demanding to compute it on general manifolds. We generalize the diffusion mean (DM) to a weighted diffusion mean (wDM), and subsequently employ methods from stochastic bridge sampling to \emph{sample} from a distribution centered at the wDM. This removes the need for expensive iterative optimization. We briefly relate the inherent stochasticity in the construction to other stochastic neural network models.

\subsection{Structure of the paper}
The paper starts with a brief account of the fiber bundle geometry we apply in the remainder of the paper. We then investigate the effect of curvature when combining multiple convolutions that each transport orientations along single geodesics. We then turn to the constructions targeting the two cases of manifold domain and manifold target, respectively. The paper ends with conclusion and outlook.
%For each case, we present the background from the geometric deep learning side, connect this with the relevant theory from fiber bundle geometry and stochastics, present the proposed constructions, and prove important properties.

The aim of the paper is to introduce methods from fibre bundle geometry and stochastics on manifolds to the geometric deep learning community from a theoretical viewpoint. We leave actual experimental validation of the methodology to future work.
While we focus on continuous manifold geometries, the methods are applicable as well to discrete geometries using discrete connections and parallel transport.

\section{Fiber bundle geometry, parallel transport and horizontality}
\label{sec:geometry}
%To further highlight the geometric setting for the use of parallel transport in the convolutions above, 
We here outline the fiber bundle geometry that we will use in the remainder of the paper. %Particularly, we use that parallel transport can be phrased as a horizontal flow in the \emph{orthonormal frame bundle} $OM$. To see this, 
Particularly, we define the frame bundle $FM$, its subbundle $OM$, and we discuss the relation between horizontality and parallel transport.
Further details on frame bundles as used here can for example be found in the books \cites{hsu_stochastic_2002,kolar_natural_1993}. The frame bundle has been used in the context of geometric deep learning before, e.g. in \cite{cohen_gauge_2019}. Here we use the frame bundle as well, however, we focus on flows in the frame bundle, i.e. the \emph{transport} of frames. Figure~\ref{fig:FM} shows the maps between the bundles used below, and Table~\ref{tbl:symbols} a list of symbols. 

\begin{center}
\begin{table}[h]
\begin{tabular}{ll}
  Symbol & Description \\
  \hline 
  $M$ & differentiable manifold \\
  $d$ & dimension of $M$ \\
  $g$ & Riemannian metric \\
  $FM$ & frame bundle of $M$ \\
  $OM$ & orthonormal frame bundle of $M$ \\
  $\pi$ & bundle map $\pi:FM\to M$ \\
  $P$, $P_{\gamma(x,u)}$ & parallel transport (along geodesic $\gamma(x,u)$) \\
  $TFM$ & tangent bundle of $FM$ \\
  $VFM$, $HFM$ & vertical, horizontal subbundles of $TFM$ \\
  $h_u$ & horizontal lift $h_u:T_{\pi(u)}M\to H_uFM$ \\
  $H_i$ & $ith$ horizontal vector field $H_i(u)=h_u(u_i)$ \\
  $\Cc$ & connection on $M$ \\
  $R$ & (Riemannian) curvature tensor \\
  $U_t$ & $OM$-valued stochastic process \\
  $W_t$ & $\R^d$-valued Brownian motion \\
\end{tabular}
\caption{List of symbols.
}
\label{tbl:symbols}
\end{table}
\end{center}

The frame bundle $FM$ of a manifold $M$ is the fiber bundle $\pi:FM\to M$ over $M$ where each element $u\in FM$ is an ordered basis for the tangent space $T_xM$, $x=\pi(u)$. The map $\pi$ attaches the base point $x=\pi(u)$ in $M$ to each element $u\in FM$. If $M$ has dimension $d$, the frame $u$ is then a $d$-tuple $(u_1,\ldots,u_d)$ of tangent vectors $u_i$ that each are elements of the tangent space $T_xM$ and that in combination constitute a basis for $T_xM$. A vector $\mathbf v\in\R^d$ can be multiplied on this basis giving an invertible linear map $\R^d\to T_xM:\mathbf v\to \sum_{i=1}^du_i\mathbf v^i$. The resulting vector in $T_xM$ is concisely written $u\mathbf v$ in accordance with the standard notation for multiplication $A\mathbf v$ of a vector $\mathbf v$ on a linear map $A$. In the following, we use boldface for \emph{Euclidean} vectors.

When $M$ is Riemannian, the subbundle of the frame bundle $FM$ consisting of \emph{orthonormal} frames is denoted the orthonormal frame bundle $OM$. Orthonormality is defined with respect to the Riemannian metric $g$ such that $g(u_i,u_j)=\delta_{ij}$, $i,j=1,\dots,d$ when $u\in OM$. Each frame $u\in FM$ provides an invertible linear map $\R^d\to T_xM:\mathbf v\to \sum_{i=1}^du_i\mathbf v^i$, and $FM$ and $OM$ can therefore be viewed as the principal fiber bundles $\GL(\R^d,TM)$, and $\O(\R^d,TM)$. $\GL(d)$ naturally acts on $FM$ on the right by $a.u\mapsto u\circ a$, $a\in\GL(d)$. The structure group is therefore $\GL(d)$. For $OM$, the structure group is the subgroup $\O(d)$.

A connection $\nabla$ on $M$, e.g., the Levi-Civita connection, lifts to a fiber bundle connection $\Cc$ on $FM$: A path $u(t)\in FM$ has zero acceleration if and only if each basis vector $u_i(t)$ is parallel transported on $M$. The connection $\Cc$ provides a split of the tangent bundle $TFM$ into \emph{vertical} and \emph{horizontal} components: The vertical component $VFM$ is the subbundle $\set{v\in TFM: \pi_*(v)=0}$, i.e., derivatives of paths $u(t)$ satisfying $u(t)=x$ for all $t$. That is, the base point is fixed and only the frame changes. The horizontal subbundle $HFM$ consists of derivatives of zero-acceleration, paths in $FM$ along which each basis vector $u_i(t)$ is parallel transported along the path $\pi(u(t))$ in $M$. Such paths are called \emph{horizontal}. The connection $\Cc$ is then explicitly a projection $TFM\to VFM$, and, using this, $TFM$ can be split into the direct sum $TFM=VFM\oplus HFM$. $HFM$ and $VFM$ being subsets of $TFM$ are also denoted the horizontal and vertical distributions, respectively.
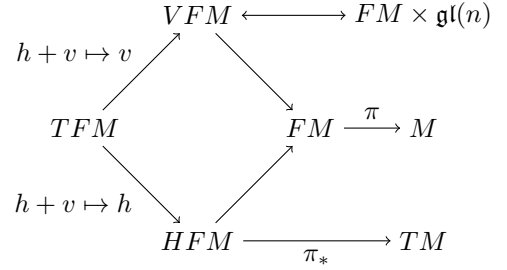
\begin{figure}[t]
\begin{center}
  \begin{tikzpicture}
  \node (TFM) {$TFM$};
  \node (empty) [right of=TFM] {};
  \node (HFM) [below of=empty] {$HFM$};
  \node (VFM) [above of=empty] {$VFM$};
  \node (FM) [right of=empty] {$FM$};
  \node (M) [right of=FM] {$M$};
  \draw[->] (TFM) to node [swap] {$h+v\mapsto h$} (HFM);
  \draw[->] (TFM) to node {$h+v\mapsto v$} (VFM);
  \node (FMlie) [above of=M] {$FM\times\mathfrak{gl}(n)$};
  \node (TM) [below of=M] {$TM$};
  \draw[->] (HFM) to node {} (FM);
  \draw[->] (VFM) to node {} (FM);
  \draw[->] (HFM) to node [swap] {$\pi_*$} (TM);
  \draw[<->] (VFM) to node [swap] {} (FMlie);
  \draw[->] (FM) to node {$\pi$} (M);
\end{tikzpicture}
\end{center}
\caption{Relations between the manifold, frame bundle, the horizontal
  distribution $HFM$, and the vertical bundle $VFM$.
  The connection $\Cc$ on $FM$ provides the splitting $TFM=HFM\oplus VFM$.
  The restrictions $\pi_*|_{H_uM}$ of the push-forward of the projection map $\pi:FM\to M$ are invertible maps $H_uM\rightarrow T_{\pi(u)}M$ with inverse $h_u$, the horizontal lift.
  The vertical bundle $VFM$ is isomorphic to the trivial bundle
  $FM\times\mathfrak{gl}(n)$.
}
\label{fig:FM}
\end{figure}

Because of this decomposition of $TFM$, any vector $v\in T_xM$ can be lifted to a unique vector in $H_uFM$, $\pi(u)=x$. This operation written as $h_u:T_{\pi(u)}M\to H_uFM$ is denoted as the \emph{horizontal lift} of $v$. In particular, the basis vectors $u_1,\ldots,u_d$ can be lifted to vectors $H_i(u):=h_u(u_i)$, $i=1,\dots,d$. This gives the set of \emph{horizontal vector fields} on $FM$. Importantly, the fields $H_i$ are globally defined, smooth, and, for each $u$, they provide a basis for $H_uFM$.

%The transport operation $P_{x\leftarrow p}$ is . For this reason, we write $P_{\gamma(x,u_xv)}^{-1}$ for $P_{x\leftarrow \Exp_x(u_xv)}$ below, where $\gamma(x,w)$ denotes the geodesic starting at $x$ with derivative $\dot{\gamma}(0)=w$. The uniqueness of the geodesic between $x$ and $p$ is only ensured away from the cut locus of the manifold. The kernel $k$ must therefore have limited support in order for the operation to be well-defined.

Let $\gamma$ be a curve in $M$. The connection $\Cc$ determines the parallel transport along $\gamma$ which we write $P_{\gamma}$. Often $\gamma$ will be a geodesic in which case we write $\gamma(x,w)$ if $\gamma$ starts at $x$ with derivative $\dot{\gamma}(0)=w$. Using the geometric setup described above,we can express parallel transport as used in %\eqref{eq:dir_conv} and 
e.g. \cite{poulenard_multi-directional_2018,cohen_gauge_2019} as a flow in $OM$:
\begin{lemma}
  Let $\mathbf v\in\R^d$, $u\in OM$ and $x=\pi(u)$. The transport $u(t)=P_{\gamma(x,tu\mathbf v)}(u)$ is an integral curve of a horizontal flow in $OM$.
  \label{lem:transport}
\end{lemma}
\begin{proof}
  For $\mathbf v\in\R^d$, $h_u(u\mathbf v)=\sum_{i=1}^dH_i(u)\mathbf v^i$ is a vector field on $FM$. Let $\Phi:FM\times\R\to FM$ be the unique flow satisfying $\partial_t\Phi(u)=h_u(u\mathbf v)$, and set $u(t):=\Phi(u,t)$. Then $u(t)$ is horizontal because $h_u(u\mathbf v)\in HFM$, and thus, for each $i=1,\dots,d$, $u_i(t)$ is parallel transported along $\pi(u(t))$. In particular, $u\mathbf v=\sum_{i=1}u_i\mathbf v^i$ is parallel transported along $\pi(u(t))$, and because $u\mathbf v=\partial_t\pi(u(t))$, $\pi(u(t))$ is the geodesic $\gamma(x,u\mathbf v)$. $u(t)$ is orthonormal for all $t$ because $u\in OM$ and the distribution $HFM$ is tangent to $OM$.
\end{proof}
While the lemma only gives an expression for the parallel transport $P_{\gamma(x,u\mathbf v)}(u)$ in the language of horizontal flows, it provides the basis for understanding the coupling between convolution layers and the stochastic horizontal flows described below.

\section{Orientations of filters and curvature}
We here aim to use the constructions outlined in section~\ref{sec:geometry} to show how curvature couples with the change of orientations happening when directions are transported backwards through a multilayer network on the evaluation of the last layer. In the next section, we will then use this approach to derive a continuous transport of orientations globally over the nonlinear domain. This allows, for example, to avoid max-pooling over directions before a fully connected final layer that combines information from distant points of the space. The overall aim is to show how the frame bundle $FM$, the subbundle of orthonormal frames $OM$, and horizontal flows in the tangent bundle $TFM$ provide a structured way to account for the change of orientations caused by the holonomy of the manifold.

We start with a brief outline of strategies for generalizing convolution to manifold domains, focusing in particular on the directional functions as introduced in \cite{poulenard_multi-directional_2018} and fiber bundles as used in the gauge equivariant networks \cite{cohen_gauge_2019}. The different formulations of convolutions discussed below are summarized in Table~\ref{tbl:convolutions}.

\begin{center}
  \setlength{\tabcolsep}{-1pt}
\begin{table}[h]
\begin{tabular}{lll}
  Definition & Description \\
  \hline 
  $k\ast f(\mathbf x)=\int_{\R^d}k(-\mathbf v)f(\mathbf x+\mathbf v)d\mathbf v$ & Euclidean, $f:\R^d\to\R$ \\
  $k\ast f(x,v)=\int_{T_xM}k_w( v)f(\overline{\Exp}_x(v))dv$ & $\Exp$+par.trans., $f:M\times TM$ \\
  \hspace{-.25cm} $\begin{array}{l}
    k\ast f(x)=\\
    \quad\int_{\R^d} k(\mathbf v)\rho_{x\leftarrow\Exp_x(u_x\mathbf v)}f(\Exp_x(u_x\mathbf v))d\mathbf v
  \end{array}$ & gauge, $f:M\to N_{\mathrm{in}}$ \\
  $k\ast f(u)=\int_{\R^d}k(-\mathbf v)f(P_{\gamma(x,u\mathbf v)}(u))d\mathbf v$ & directional, $f:OM\to\R$ \\
  \hspace{-.25cm} $\begin{array}{l}
    k\ast_{U_T^u,d\mathbf v} f(x)=\\
    \ \int_{\R^d}\int_{\pi^{-1}(\Exp_x(u\mathbf v))}k(-\mathbf v)f(U_T^u)P(dU_T^u)d\mathbf v
  \end{array}$ & $OM$ distribution $U_T^u$ \\
  \hspace{-.25cm} $\begin{array}{l}
    k\ast_{U_T^u,\Log} f(x)=\\
    \quad\E[k(-u^{-1}\Log_x(\pi(U_T^u)))f(U_T^u)]
  \end{array}$ & $OM$ distribution $U_T^u$, $\Log$ \\
  \hspace{-.25cm} $\begin{array}{l}
    k\ast_{W_T} f(u)=
    \E[k(-W_T)f(U_T^u)]
  \end{array}$ & $\R^d$ Brown. motion $W_t$, \\
  & $OM$ process $U_t^u$\\
%  $\ast_{W_T}$ & & convolution driven by Brownian motion (see \eqref{eq:antidev_conv}) \\
\end{tabular}
\caption{List of convolution operators used in the paper.
}
\label{tbl:convolutions}
\end{table}
\end{center}

\subsection{Background}
%{\color{red} TODO: papers to mention: spectral methods: \cite{bruna_spectral_2013}; 
%\cite{schonsheck_parallel_2018} simple parallel transport of filter orientations;
%group invariant \cite{cohen_group_2016};
%anisotropic \cite{sommer_modelling_2017}, \cite{sommer_anisotropically_2016}; }
%\cite{masci_geodesic_2015} GCC, patch operator
%\cite{bronstein_geometric_2017}
%\cite{boscaini_learning_2016} patch operator from anisotropic heat equation using rotation with respect to the maximum curvature direction (we can use our rotation here if needed)
%\cite{monti_geometric_2016} pseudo coordinates
%\cite{cohen_gauge_2019}
%\cite{poulenard_multi-directional_2018} parallel transport, directions to layers, does not solve cut locus problem (perhaps has nice notation)

Let $M$ be a $d$-dimensional Riemannian manifold and $f:M\rightarrow\R^\din$ a vector-valued function, e.g. an single channel image ($\din=1$) or an RGB image ($\din=3$). 
If $M$ is Euclidean, i.e. $M=\R^d$, each layer in a convolutional network applies the Euclidean convolution $k\ast f(x)=\int_{\R^d}k(-\mathbf v)f(x+\mathbf v)d\mathbf v$ using a kernel $k:\R^d\to\R^{\dout\times \din}$ to give a $\dout$-dimensional output. This is followed by composition with a non-linearity on each component. Discretizing the convolution spatially gives the output as $y(x)=\sum_{i,j}  k(-i,-j) f(x+(i,j))$ when $d=2$ and the sum over $i$ and $j$ runs over the support of the kernel. The kernel $k$ is then specified by a finite set of entries referred to as weights. The linearity of the convolution operation gives rise to the view of each layer as a tensor on functions $M\to\R^\din$. When the convolution appears in the $l$-th layer of a multilayer network, each component $y^n$ of the vector-valued output can be regarded as a result of a tensor convolution 
\begin{equation}
  y^n=k^n_1\ast f^1+\cdots+k^n_m\ast f^m
  \label{eq:tensor_conv}
\end{equation}
using a set of kernels $k^n_m$.

When $M$ is a nonlinear manifold, approaches for generalizing convolution includes spectral methods \cite{bruna_spectral_2013}, and techniques using the group structure when $M$ is a Lie group or a homogeneous space \cite{cohen_group_2016}.
\cite{monti_geometric_2016}, building on \cite{masci_geodesic_2015} and \cite{boscaini_learning_2016}, defines the convolution operator using \emph{pseudo-coordinates}, a family of local charts $\phi_x$, $x\in M$ that by mapping each point in a neighborhood $U_x\subset M$ to $\R^d$ allows the notion of \emph{patch-operator} to be defined as $D_jf=\int_{U_x}w_j(\phi_x(y))f(y)dy$. The patch operator is subsequently matched to a template to give the final generalized convolution. Particularly, the patch operator can be chosen to be rotationally invariant, e.g. using local geodesic polar coordinates, thus removing ambiguity in the orientation of the chart. However, this significanty restricts the wealth of kernels that can be used in the network. 

The pseudo-coordinates in \cite{boscaini_learning_2016} allows rotationally non-invariant kernels by aligning orientations with respect to the directions of maximal curvature direction.
%One choice of $\phi_x$ is a local inverse of the Riemannian exponential map $\Exp_x$, i.e. the Riemannian log map $\Log_x$, often written $\vec{xy}:=\Log_x(y)$, composed with the inverse of a linear isometry $u:R^d\to T_xM$. Such a combination $y\to u^{-1}(\vec{xy})$ gives the notion of geodesic convolutional networks \cite{masci_geodesic_2015}. However, the isometry can be found from the Riemannian metric only up to rotation, and the integral in $f\ast_{\mathrm{patch}} k(x)$ is therefore only well-defined for rotationally invariant kernels $k$. 
However, handling the ambiguity of rotations is not solved entirely in this way because maximal curvature direction may not be defined (e.g. on constant curvature spaces such as spheres); curvature is a local notion which can imply rapid shifts in directions over short distances; topology constrains the set of non-vanishing continuous vector fields (e.g., the hairy-ball theorem on spheres) and so a continuous set of orientations cannot generally be found on topologically non-trivial spaces.

To handle the lack of global orientations on surfaces ($d=2$), \cite{poulenard_multi-directional_2018} proposes to convolve with functions $f:M\times TM$ that in the second argument take a tangent vector representing a direction. This vector is parallel transported along minimizing geodesics resulting in the convolution $k\ast f(x,w)=\int_{T_xM}k_w( v)f(\overline{\Exp}_x(v))dv$ where the map $\overline{\Exp}_x$ is the Riemannian exponential map $\Exp_x:T_xM\to M$ combined with parallel translation of the (normalized) vector $v$ to provide directional information for the evaluation of $f$. The direction $w$ in the 2D case determines an isometry between $T_xM$ and $\R^2$, and the kernel $k_w$ is defined using this isometry. The use of directional functions implies that directions are propagated between layers in a consistent way. Parallel transport is also used to define convolution in \cite{schonsheck_parallel_2018}.
%Poulenard et al. writes this using push-forwards giving $f*k(x,v)=\ip{(\overline{\Exp}_x)^*f,(\tau_x)^*f}_{L^2(\R^d}$ where $\tau_x$ is the regular translation $\tau_x(y)=x-y$.

Gauge equivariant networks \cite{cohen_gauge_2019} provide a related approach to handle directional ambiguity. A gauge for the tangent bundle $TM$ is a map $u:U\times \R^d\to TM$ that for each $x$ in an open set $U\subset M$ gives an invertible linear map $u_x:\R^d\to T_xM$. Equivalently, a gauge is a local section of the frame bundle. Let $\pi_{\mathrm{in}}:N_{\mathrm{in}}\to M,\pi_{\mathrm{out}}:N_{\mathrm{out}}\to M$ be two vector bundles over $M$. A gauge equivariant convolution takes as input a section $f:M\to N_{\mathrm{in}}$ of $N_{\mathrm{in}}$ and outputs a section of $N_{\mathrm{out}}$ given by $k\ast f(x)=\int_{\R^d} k(\mathbf v)\rho_{x\leftarrow\Exp_x(u_x\mathbf v)}f(\Exp_x(u_x\mathbf v))d\mathbf v$ where $P_{x\leftarrow p}:\pi_{\mathrm{in}}^{-1}(p)\to\pi_{\mathrm{in}}^{-1}(x)$ is a transport operation in $N_{\mathrm{in}}$ along a unique minimizing geodesic, either by parallel translating each vector of $u_{x,\mathrm{in}}$ to $p$ or by using a connection on the bundle $N_{\mathrm{in}}$. Gauges enter in the kernel as $k(\mathbf v)=u_{x,\mathrm{out}}k(\mathbf v)u_{x,\mathrm{in}}^{-1}$ where $k(\mathbf v)\in\R^{\din,\dout}$ and $u_{x,\mathrm{in}}$, $u_{x,\mathrm{out}}$ are frames for $N_{\mathrm{in}},N_{\mathrm{out}}$, respectively.
The convolution can be shown to be independent of the choice of gauges if $k$ satisfies an invariance condition \cite{cohen_gauge_2019} dependent on the structure group, see also \cite{kondor_generalization_2018}. Particularly relevant is equivariance to the rotation group $\SO(d)$, equivalently choices of orthonormal frames in the bundle $OM$ described below. In this case, gauge equivariance for scalar valued functions is equivalent to rotational invariance.

\subsection{Directional functions}
A natural generalization to higher dimensions of the convolution of directional functions on surfaces in \cite{poulenard_multi-directional_2018} is to define convolution on functions $f:OM\to\R$ where $OM$ is the orthonormal frame bundle $OM$. Then convolution can be defined as  
\begin{equation}
  k\ast f(u)
  =
  \int_{\R^d}k(-\mathbf v)f(P_{\gamma(x,u\mathbf v)}(u))d\mathbf v,\ x=\pi(u)
  \label{eq:dir_conv}
\end{equation}
Note that the frame bundle element $u$ is used to map the $\R^d$ vector $v$ to the vector $u\mathbf v$ in $T_xM$ to give the direction of the geodesic $\gamma(x,u\mathbf v)$. The frame $u$ is then parallel transported along this geodesic to enable evaluation of $f$.
Here, we focus on real valued functions and kernels $k:\R^d\to\R$ though the construction and the additional convolutions defined below extend to multiple output channels or bundle valued outputs (see section~\ref{sec:conv_properties}).

\begin{remark}
  Note the difference between this construction and the gauge equivariant case: The input functions in the latter case have no directional input information; instead they are equivariant to gauge changes, e.g., the action of $\SO(d)$. However, for $f:OM\to\R$ and $x\in M$, we can view the restriction $f|_{\pi|_{OM}^{-1}(x)}$ to the fiber over $x$ as an element of the bundle $N=\set{x\in M|\ h:\pi|_{OM}^{-1}(x)\to\R}$ of functions on the fibers $\pi|_{OM}^{-1}(x)$. The group $\SO(d)$ acts on $N$ on the right by $a.h(u)=h(a.u)$, $a\in \SO(d)$, and the convolution \eqref{eq:dir_conv} can be seen as a gauge equivariant network with $N_{\mathrm{in}}=N_{\mathrm{out}}=N$. In fact, in this case, equivariance implies that any convolution output is a function on the fibers because of the dependence on the frame/gauge. Orientation functions can be seen as continuous analogues of the discrete rotations used in \cite{cohen_gauge_2019}.
\end{remark}

In the convolution \eqref{eq:dir_conv}, it is important to note that directional information propagates \emph{backwards} through a composition of layers: As illustrated in Figure~\ref{fig:backwards}, let $f_1=k_1\ast f$ and $f_2=k_2\ast f_1$. Then in the convolution to produce $f_2(u)$, $u$ is parallel transported along geodesics from $x=\pi(u)$ in order to evaluate $f_1$. The frames $P_{\gamma(x,u\mathbf v)}(u)$ are then in turn parallel transported a second time before evaluating $f$. Because of the path dependence of parallel transport, this is in general not equal to parallel transporting only once if evaluating a filter $(k_2\ast k_1)\ast f$ with $k_2\ast k_1$ denoting the standard Euclidean convolution. We show below how this difference is related to the curvature of $M$.
\begin{figure}[t]
  \centering
  \def\svgwidth{1.\columnwidth}
  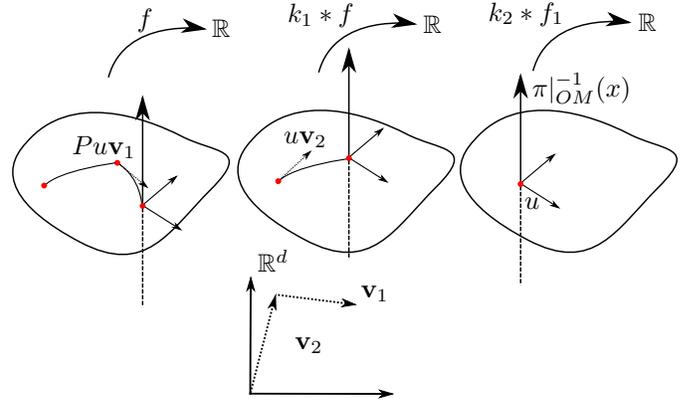 
  \caption{
    Directional information captured in the frame $u\in OM$ which is an element of the fiber $\pi_{OM}^{-1}(x)$ (illustrated by vertical arrow in rightmost sketch) over a point $x\in M$ (red dot). In the convolution \eqref{eq:dir_conv}, $u$ is parallel transported along geodesics with initial direction $u\mathbf v_2$ for vector $\mathbf v_2\in\R^d$ (center). When composing convolutions, the frame $P_{\gamma(x,u\mathbf v_2)}u$ at $\gamma(x,u\mathbf v_2)$ is transported along a second geodesic with direction $P_{\gamma(x,u\mathbf v_2)}u\mathbf v_1$ (left, subscript on $P_{\gamma(x,u\mathbf v_2)}$ omitted). The directional information in $u$ is thus transported \emph{backwards} through the layers of a multiplayer network. The curvature implies that this transport is path dependent: A different choice of path would yield a different transport.
        }
  \label{fig:backwards}
\end{figure}

The convolutions defined in \cite{poulenard_multi-directional_2018}, \cite{cohen_gauge_2019} and \eqref{eq:dir_conv} above implicitly construct a gauge in the evaluation of the integral in the convolution because the parallel transport $P_{\gamma(x,u\mathbf v)}(u)$ gives a local section of $OM$. This is a specific choice of gauge, and a different choice would result in different results of the convolution. In particular, a different choice of paths along which $u$ is parallel transported would have this effect (see also discussion in \cite{cheng_covariance_2019}). Below, we will embrace this by defining a measure on such paths and integrating out the effects of the difference in parallel transport. As noted above, $k$ is often assumed to have limited support implying that the choice of paths may not have a great effect. However, this may not be the case when the output features of distant points are compared in fully connected layers appearing as the last layers of a multilayer network. Currently, max-pooling over directions is often applied before such a layer \cite{cohen_gauge_2019}. Below, we construct a principled way to integrate rotations without  such a pooling.

\subsection{Curvature and composition of layers}
Associativity $(k_2\ast k_1)\ast f=k_2\ast (k_1\ast f)$ of the Euclidean convolution is a consequence of its translation invariance. Parallel transport of frames along geodesics implies translation invariance along those geodesics, but the path dependence rules out translation invariance between points when the path is not specified. For one layer filters with limited support, one can reasonably restrict to transport along geodesics as above. But with multiple layers, the integral in the convolution is evaluated along compositions of geodesics giving curves that are only piecewise geodesic as illustrated in Figure~\ref{fig:backwards}. This again implies a rotation of filters. We here connect this fact to horizontal $OM$ flows and curvature.

The following statement which comes from an application of Taylor's theorem on $OM$, uses parallel transport and $OM$ to express the lack of associativity and commutativity of the convolution operation directly in terms of the curvature of $M$. We use the bracket $[h_u(v),h_u(w)]f=h_u(v)h_u(w)f-h_u(w)h_u(v)f$ for $v,w\in T_xM$ which measures the lack of commutativity of derivatives by horizontal vector fields. The bracket is directly linked to the curvature tensor $R$ of $M$ by the relation
\begin{equation}
  R(v,u)=-\Cc([h_u(v),h_u(w)])
   \ ,
  \label{eq:curvature_bracket}
\end{equation}
i.e., the curvature measures the vertical component of the bracket between horizontal vector fields.
\begin{theorem}
  Let $k_1,k_2$ be kernels with $\supp(k_i)\subseteq B_r(0)$, and $f\in C^3(OM,\R)$. Then
  \begin{align}
    &k_2\ast (k_1\ast f)
    -
    (k_2\ast k_1)\ast f
    \label{eq:associativity}
    \\
    &\ =
    \int_{\R^d}
    \int_{\R^d}
    k_2(-\mathbf v_2)k_1(-\mathbf v_1)
    h_u(u\mathbf v_2)h_u(u\mathbf v_1)f
    d\mathbf v_1
    d\mathbf v_2
    \\&\quad\ 
    +
    o(r^{d+1})
    \nonumber
  \end{align}
  and
  \begin{align}
    &k_2\ast (k_1\ast f)
    -
    k_1\ast (k_2\ast f)
    \label{eq:commutativity}
    \\
    &\ =
    \int_{\R^d}
    \int_{\R^d}
    k_2(-\mathbf v_2)k_1(-\mathbf v_1)
    [h_u(u\mathbf v_2),h_u(u\mathbf v_1)]f
    d\mathbf v_1
    d\mathbf v_2
    \\&\quad\ 
    +
    o(r^{d+1})
    \nonumber
    \ .
  \end{align}
  In particular, the vertical part of the non-commutativity in \eqref{eq:commutativity} is a function of the curvature tensor $R(u\mathbf v_2,u\mathbf v_1)$.
  \label{eq:assoc_lem}
\end{theorem}
\begin{proof}
  Let $f^\mathbf v:OM\to\R$ be the map $u\mapsto f(P_{\gamma(\pi(u),u\mathbf v)}u)$ (the map in the integrand of \eqref{eq:dir_conv}).
  By Taylor's theorem and Lemma~\ref{lem:transport}, $f^\mathbf v(u)=f(u)+h_u(u\mathbf v)f+o(\|\mathbf v\|)$. Applying Taylor's theorem again, we get $(f^{\mathbf v_1})^{\mathbf v_2}(u)=f(u)+h_u(u{\mathbf v_1})f+h_u(u{\mathbf v_2})f+h_u(u{\mathbf v_2})h_u(u{\mathbf v_1})f+o(\|\mathbf v_1\|,\|\mathbf v_2\|)$. Then, using the regular Euclidean convolution $k_1\ast k_2$,
  \begin{align*}
    &k_2 \ast (k_1\ast f)(u)
    -(k_2 \ast k_1)\ast f(u)
    \\
    &=
    \int_{\R^{2d}}
    k_2(-\mathbf v_2)k_1(-\mathbf v_1)
    (
    (f^{\mathbf v_1})^{\mathbf v_2}(u)
    -
    f^{\mathbf v_1+\mathbf v_2}(u)
    )
    d(\mathbf v_1,\mathbf v_2)
    \\
    &=
    \int_{\R^{2d}}
    k_2(-\mathbf v_2)k_1(-\mathbf v_1)
    h_u(u\mathbf v_2)h_u(u\mathbf v_1)f
    )
    d(\mathbf v_1,\mathbf v_2)
    \\&\quad
    +
    o(r^{d+1})
    \ .
  \end{align*}
  The commutativity relation \eqref{eq:commutativity} results from using \eqref{eq:associativity} and commutativity of the Euclidean convolution.
\end{proof}

The result makes explicit the relation between non-commutativity and non-associativity of convolution kernels when using parallel transport and the curvature of the manifold, equivalently non-integrability of the horizontal distribution $HFM$ as seen by the brackets $[h_u(u\mathbf v_1),h_u(u\mathbf v_2)]$ being nonzero.
One can also view the use of the Riemannian exponential map and parallel transport along geodesics as a linearization of the manifold \cite{sommer_manifold_2010}. The lack of commutativity is a reflection of the fact that such linearizations generally do not lift to subspaces of the frame bundle: The lifted vector fields are horizontal, but the horizontal fields are only integrable (and thus tangent fields of a subspace) when the curvature is zero.

Theorem~\ref{eq:assoc_lem} is local statement that describes how curvature at the center point $x$ implies noncommutativity of the convolution. If the kernels have significant mass away from $0$ in $\R^d\simeq T_xM$, this linearized view cannot fully capture the nonlinear effect of curvature. The resulting deviation is captured in the remainder term $o(r^{d+1})$. Importantly, the theorem states that the convolution will be noncommutative no matter how small $r$ is when $k_1,k_2$ are positive in the ball $B_r(0)$, regardless of the dimension $d$.

\section{Convolution with horizontally distributed orientations}
When applying multiple convolutions, the consecutive application of parallel transport is related to time-parameterized flows in stochastic analysis on manifolds. We first describe the Riemannian Brownian motion as an example of such a flow and use it to distribute orientations along multiple paths between pairs of points. Subsequently, we show how compositions of many layers can be seen as a time-discretized $OM$ flow.

Let $(\Omega,(\mathcal F_t)_{t\ge 0},\mathbb P)$ be a standard probability space with filtration $(\mathcal F_t)_{t\ge 0}$. The Riemannian Brownian motion is a stochastic process $X_t$, i.e. a time-indexed sequence of $\mathcal F_t$ measurable, $M$-valued random variables, that has density $p_t(\cdot; x)$ where $x$ denotes the starting point of the process, i.e. $X_0=x$. $p_t$ is also denoted as the heat kernel, and it satisfies the heat equation $\partial_t p_t=-\frac12 \Delta p_t$ where $\Delta$ is the Laplace-Beltrami operator of $M$. $p_t$ is smooth for all $t>0$, and non-zero if $M$ is connected. The heat flow has been used extensively in geometric deep learning \cite{bruna_spectral_2013}. Here we focus on its relation to parallel transport.

The construction of the Riemannian Brownian motion is non-trivial due to the global nature of the process but the local nature of charts, see e.g. \cite{emery_stochastic_1989}. One approach is the Eells-Elworthy-Malliavin construction \cite{elworthy_geometric_1988} that avoids the use of charts by mapping horizontal $OM$ flows to $M$: An $\R^d$-valued Euclidean Brownian motion $W_t$ is mapped to an $OM$-valued stochastic process $U_t$ by the SDE $dU_t=\sum_{i=1}^d H_i\circ_{\mathcal S} dW_t^i$ where $\circ_{\mathcal S}$ denotes Stratonovich integration, see e.g. \cite{hsu_stochastic_2002} for details. The starting point $U_0$ is one point in $u\in OM$. We make the dependence on the starting point explicit by writing $U_t^u$. By mapping $U_t^u$ to the manifold, the resulting process $X_t^x=\pi(U_t^u)$ is a Brownian motion with starting point $x=\pi(u)$. Figure~\ref{fig:processes} illustrates the relation between $W_t$, $U_t$ and $X_t$. Long-time existence of the Brownian motion can be proven under mild assumptions on $M$ (e.g. compactness is sufficient).
\begin{figure}[t]
  \centering
  \def\svgwidth{1.\columnwidth}
  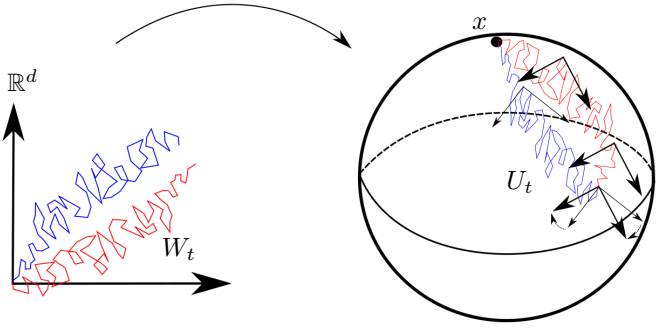 
  \caption{
    The relation between the Euclidean $\R^d$-valued Brownian motion $W_t$, the $OM$ process $U_t$ that carries the frame by parallel transporting along the stochastic paths $X_t=\pi(U_t)$ on $M$. $W_t$ is mapped to $U_t$ and $X_t$ by \emph{development}. The reverse mapping is denoted \emph{anti-development}. Two sample paths $X_t(\omega_1)$, $X_t(\omega_2)$ (blue/red) ending at the same point in $M$ ($X_T(\omega_1)=X_T(\omega_2)$) need not end at the same point when anti-developed to $\R^d$. The curvature implies that $U_T(\omega_1)$, $U_T(\omega_2)$ may hit different points in the fiber over the endpoint. The difference is a rotation (or gauge transformation). 
        }
  \label{fig:processes}
\end{figure}

For each $t$, $U_t$ is an $OM$-valued random variable and $X_t$ is an $M$-valued random variable. The distribution corresponding to $X_t$ has density $p_t(\cdot; x)$. $U_t$ may also have a smooth density on $OM$, however, this depends on the curvature: $U_t$ may at time $t$ hit a fiber $\pi|_{OM}^{-1}(y)$, $y\in M$ along many different paths on $M$, not just geodesics. Each such path will have its own parallel transport, and which point in the fiber is hit depends on the path. The difference will be a shift of orientation, i.e., a gauge transformation. $U_t$ thus gives a distribution of orientations for each fiber $\pi|_{OM}^{-1}(y)$. For flat manifolds, parallel transport is path independent, and $U_t$ is supported on a $d$-dimensional submanifold of $OM$. Conversely, on curved manifolds with holonomy group $\SO(d)$, all rotations appear, $U_t$ will be non-zero on all of $OM$, and it will have a smooth, positive density.

While for the convolution \eqref{eq:dir_conv} we only used geodesics from a point $x$ in the parallel transport, the $OM$-flow defines a probability measure $\mathbb P_{U_t^u}$ on stochastic paths from $x$. We can use this to define a convolution that takes any $\mathbb P_{U_t^u}$-measurable path into account:
  \begin{equation}
    k\ast_{U_T^u,d\mathbf v} f(u)
    =
    \int_{\R^d}\int_{\pi^{-1}(\Exp_x(u\mathbf v))}k(-\mathbf v)f(U_T^u)\mathbb P(dU_T^u)d\mathbf v
    \label{eq:horz_conv}
  \end{equation}
  Note that the definition integrates over each fiber $\pi^{-1}(\Exp_x(u\mathbf v))$ in $OM$ with the distribution on each fiber being a result of the $OM$ flow $U_t^u$ at a fixed time $T>0$, see Figure~\ref{fig:horz_conv}. The notation $\ast_{U_T,d\mathbf v}$ makes the dependence on the measures $\mathbb P(dU_T^u)$ and $d\mathbf v$ in the integration explicit.
%In particular, the function $v\mapsto \int_{\pi^{-1}(\Exp_x(u\mathbf v))}k(v)f(U_T)\mathbb P_{U_t}$ that is integrated in \eqref{eq:horz_conv} is a smooth function of $v$ because the Brownian flow ensures that the orientations are continuously distributed over $OM$. 
\begin{figure}[t]
  \centering
  \def\svgwidth{.6\columnwidth}
  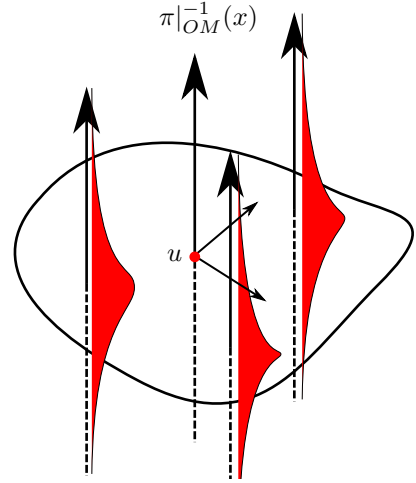 
  \caption{
     The convolution \eqref{eq:horz_conv} uses the $OM$ distribution $U_T^u$, $T>0$ that in each fiber has a distribution of frames, i.e. a distribution of orientations (red fiber density illustration). With zero curvature, parallel transport is path independent, and the distribution will be singular supported at a single element of each fiber. With curvature, the fiber distributions widen, possibly filling the fibers. For analytic manifolds, $U_T^u$ has smooth, positive density for any $T>0$ on a submanifold of $OM$.
        }
  \label{fig:horz_conv}
\end{figure}

For kernels with limited support, minimizing geodesics are in practice unique, and parallel transport can reasonably be performed along minimizing geodesics. This is generally the case for convolutional layers. However, the last layers of a network can be fully connected, and one thus cannot limit the support of the kernel. Instead, pooling over rotations can be performed to remove the rotational ambiguity of non-unique geodesics. In contrast, the convolution \eqref{eq:horz_conv} allows a fully connected layer to include information from the entire manifold while handling rotations in a principled way. In particular, the orientation distribution is continuous as a function of $u$ when $M$ is analytic as discussed below.

\subsection{Multilayer convolution as stochastic flow}
We now aim to improve the convolution \eqref{eq:horz_conv} to express it in stochastic terms, and to give it a natural behaviour when composing layers. For this, we need the notion of stochastic development and anti-development that expresses the relation between the processes $W_t$, $U_t$, and $X_t$. First, let us rewrite \eqref{eq:horz_conv} to avoid the split between the fiber integration $\mathbb P(dU_T^u)$ and the Euclidean integration $d\mathbf v$. We do that using the distribution $U_T^u$ directly:
\begin{equation}
  \begin{split}
  &k\ast_{U_T^u,\Log} f(u)
  =
  \\&\qquad=
  \int k(-u^{-1}\Log_x(\pi(U_T^u)))f(U_T^u)\mathbb P(dU_T^u)
  \\&\qquad=
  \E[k(-u^{-1}\Log_x(\pi(U_T^u)))f(U_T^u)]
  \end{split}
  \label{eq:Log_conv}
\end{equation}
where $\Log_x$ denotes the local inverse of the $\Exp_x$, and $\E$ the expectation with respect to the law of $U_T^u$. $\Log_x$ here provides pseudo-coordinates:
Because the integration is now over $OM$, we need to map to $\R^d$ to evaluate the kernel $k$. This again makes the integrand discontinuous because of the local nature of $\Log_x$ (the logarithm is discontinuous when crossing the cut locus). It turns out that this deficiency can be removed.

Recall the connection between the $\R^d$-valued Brownian motion $W_t$ in the Eells-Elworthy-Malliavin construction and the $M$-valued process $X_t=\pi(U_t)$. $X_t$ is denoted the \emph{stochastic development} of $W_t$, since $X_t$ is developed, or rolled-out, over $M$ following the stochastic increments $dW_t$ for each time $t$ using the current values of the process $U_t$ to map from $\R^d$ to $T_{X_t}M$. The reverse is also true: Any $M$-valued semimartingale can be \emph{anti-developed} to a semi-martingale on $\R^d$. Figure~\ref{fig:processes} illustrates the relation between sample paths $W_t(\omega)$ and the developments $U_t(\omega)$. Using this relation directly, we can define a convolution as
\begin{equation}
  k\ast_{W_T} f(u)
  =
  \int k(-W_T)f(U_T^u)\mathbb P(dW_t)
  =
  \E[k(-W_T)f(U_T^u)]
  \label{eq:antidev_conv}
\end{equation}
where the mapping from $W_t$ to $U_t^u$ by development is used, and the expectation on the right-hand side is with respect to the law of the Brownian motion $W_t$. Note that $k$ is evaluated on $W_T$ for a fixed $T$. This Euclidean random variable is in fact normally distributed. However, $f$ is evaluated on the $OM$-valued random variable $U_T^u$ that automatically includes directional information. In comparison with \eqref{eq:Log_conv}, $\Log_x$ is not used in \eqref{eq:antidev_conv}.

The convolution depends on the Brownian motion $W_t$ up until the evaluation time $T>0$. Varying $T$ will change the distribution of orientations over $M$: For $T$ large, all orientations will diffuse to be equally probable; in the limit $T\to 0$, the convolution \eqref{eq:dir_conv} is recovered because the $U_T^u$ measure concentrates around the points in each fiber that corresponds to parallel transport along geodesics from $x$ (see small-time asymptotic limit results in, e.g., \cite{hsu_stochastic_2002} and \cite{sommer_modelling_2017}).
\begin{remark}
  The $OM$ endpoint $U_T^u$ is dependent on the entire path $U_t^u$, $t\in[0,T]$: Let $\omega^1,\omega^2$ be two elements of $\Omega$ such that $W_T(\omega^1)=W_T(\omega^1)$. Then $U_T^u(\omega^1)$ does not necessarily equal $U_T^u(\omega^2)$. This is a consequence of curvature and reflects that development is a map from the path space $W([0,T],\R^d)$ to the path space $W([0,T],OM)$, i.e. the endpoint $U_T^u(\omega)$ is dependent on the entire path $W_t(\omega)$. The path spaces are Wiener spaces of continuous paths on $[0,T]$.
\end{remark}
\begin{figure}[t]
  \centering
  \def\svgwidth{1.\columnwidth}
  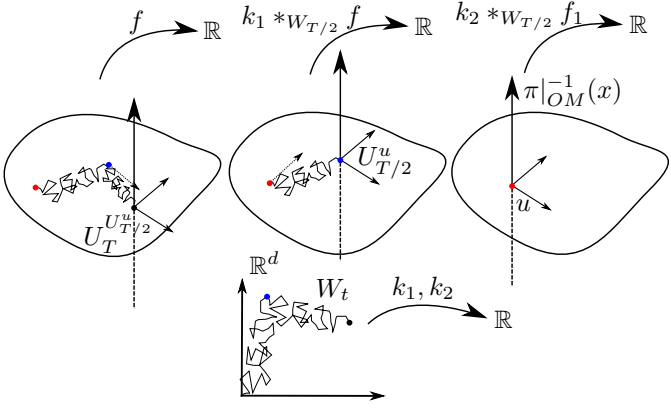 
  \caption{
    The convolution \eqref{eq:antidev_conv} applies kernels, here $k_1,k_2$, on the antidevelopment $W_t$ whereas $f$ is applied on the $OM$ process $U_T^u$. Compare with Figure~\ref{fig:backwards}.
        }
  \label{fig:antidev_conv}
\end{figure}

Brownian motion or, equivalently, the heat flow is a semi-group which is often expressed in terms of the density: $p_{t+s}(y; x)=\int_M p_s(y; z)p_t(z; x)dz$. In other words, we can obtain $W_{t+s}$, $U_{t+s}^u$, and $X_{t+s}^x$ by starting the stochastic processes at $0$, $u$, and $x$ respectively, running the process to time $t$, and then restart the processes at $W_t$, $U_t^u$ and $X_t^x$ to obtain $W_{t+s}=W_s^{W_t}$, $U_{t+s}^u=U_s^{U_t^u}$, and $X_{t+s}^x=X_s^{X_t^x}$. We can use this and development to express compositions of convolution as one integral over the Brownian motion with the filters applied at discrete times. To see this, compose two $W_{T/2}$ convolution layers to get
\begin{equation}
  \begin{split}
    &k_2\ast_{W_{T/2}} (k_1\ast_{W_{T/2}} f)(u)
  \\&\qquad=
  \E[k_2(-W_{T/2})(k_1\ast_{W_{T/2}} f)(U_{T/2}^u)]
  \\&\qquad=
  \E[k_2(-W_{T/2})\E[k_1(-(W_T-W_{T/2}))f(U_T^{U_{T/2}^u})]]
  \\&\qquad=
  \E[k_2(-W_{T/2})k_1(-(W_T-W_{T/2}))f(U_T^u)]
  \end{split}
  \label{eq:two_conv}
\end{equation}
using the semigroup property for $W_t$ and $U_t$. Note that $W_T-W_{T/2}$ is Gaussian distributed with variance equal to $W_{T/2}$. Thus, with $n$ layers and filters $k_1,\dots k_n$,
\begin{equation*}
  \begin{split}
    &k_n\ast_{W_{T/n}} (k_{n-1}\ast_{W_{T/n}} \cdots (k_1\ast_{W_{T/n}} f))(u)
  \\&\qquad=
  \E[k_n(-W_{T/n})k_{n-1}(-(W_{2T/n}-W_{T/n}))\cdots \qquad
  \\&\qquad\qquad\qquad\qquad\quad\quad\ \ \ 
  k_1(-(W_T-W_{(n-1)T/n}))f(U_T^u)]
  \end{split}
  %\label{eq:multilayer_conv}
\end{equation*}

For the evaluation of the output layer at $u$, the result of the convolution, with $f$ being the input function, the stochastic flow visits the layers evaluating $k_n$ at $t=T/n$ and $k_1$ at $t=T$. \emph{Forward} time of the processes $W_t$ and $U_t^u$ thus implies \emph{backwards} visits through the layers. All differences $W_{iT/n}-W_{(i-1)T/n}$ are normally distributed in $\R^d$. The base points $\pi(U_T^u)$ in $M$ follows the distribution of a Brownian motion started at $\pi(u)$ evaluated at $t=T$, and orientations are distributed in $OM$ by parallel translating along the stochastic paths $\pi(U_t)$ in $M$. The effect of the convolution can be seen by comparing Figure~\ref{fig:antidev_conv} with Figure~\ref{fig:backwards}.

\subsection{Properties}
\label{sec:conv_properties}
\emph{Tensor convolutions}:
When the convolution \eqref{eq:antidev_conv} appear in the $l$th layer of a multilayer network with multi-dimensional input and output, we can generalize the tensor convolution \eqref{eq:tensor_conv} by writing \eqref{eq:antidev_conv} in the form
\begin{align*}
  &y^n
  =\E[k^n(-W_t)f(U_T^u)]
  \\&\qquad
  =\E[k^n_1(-W_t)f^1(U_T^u)]+\cdots+\E[k^n_m(-W_t)f^m(U_T^u)]
\end{align*}
with a set of kernels $k^n_m$ being the entries of the kernel tensor $k$. The linearity of the expectation thus implies that convolution can be extended to tensor convolution similarly to the Euclidean case.

\emph{Equivariance}: Let $f:OM\to\R$ and $a\in\O(d)$. As mentioned previously, $a$ acts on the right on $f$ by $(a.f)(u)=f(a.u)$ (recall that $\GL(d)$ acts on $FM$ by right composition). Then
\begin{equation*}
  \begin{split}
  k\ast_{W_T} (a.f)(u)
  &=
  \E[k(-W_T)f(U_T^u\circ a)]
  =
  \E[k(-W_T)f(U_T^{u\circ a})]
  \\&=
  k\ast_{W_t^T} f(a.u)
  =
  a.(k\ast_{W_T} f)(u)
  \end{split}
\end{equation*}
because the parallel transport in $U_t^u$ acts on $u$ by composition on the left. The horizontal convolution is thus equivariant to the $\O(d)$ action on functions $OM\to\R$.

\emph{Smoothness}: When the finite bracket span of $H_uFM$, i.e. the span of iterated brackets $[[[H_{i_1}(u),H_{i_2}(u)],H_{i_3}(u)],\ldots,H_{i_r}(u)]$, $r\in\mathbb{N}$, $i_j=1,\dots,d$, generates a subspace of $T_uOM$ of constant rank as a function of $u$, there exists a submanifold of $OM$ on which $U_t^u$ has a smooth, positive density for all $t$ by the Frobenius theorem. In this case, the integrand $f(U_T^u)$ in the convolution inherits the smoothness of $f$. This is in contrast to the parallel transport of frames along minimizing geodesics where the minimizing geodesics shift discontinuously when crossing the cut locus. The constant rank condition is for instance satisfied for analytic manifolds \cite{montgomery_tour_2006}
and homogeneous spaces.

\emph{Nonlinearities}: With addition of layer-wise nonlinearities $\phi_i$, $i=1,\dots,n$, the full network takes the form
\begin{equation*}
  \begin{split}
    &\phi_n(k_n\ast_{W_{T/n}} \phi_{n-1}(\cdots \phi_1(k_1\ast_{W_{T/n}} f))(u)
  \\&\ =
  \phi_n(\E[k_n(-W_{T/n})\phi_{n-1}(\E[k_{n-1}(-(W_{2T/n}-W_{T/n}))\cdots \qquad
  \\&\ \qquad\qquad
  \phi_1(\E[k_1(-(W_T-W_{(n-1)T/n}))f(U_T^u)])])])
  \end{split}
%  \label{eq:multilayer_conv_nonlinearities}
\end{equation*}

\emph{Spatial pooling and Gaussian weighting}:
  There is an implicit Gaussian weighting in the integral in \eqref{eq:antidev_conv} since $W_T$ is normally distributed. This is in contrast to the most standard form of convolution where the integral is taken with respect to the Lebesgue measure on $\R^d$. This can be compensated for in reweighting the kernel, i.e. exchanging $k(x)$ with $k(x)/p_T(x)$ where $p_T$ is the density of the centered normal distribution in $\R^d$ with variance $T$. The use of the Brownian motion makes the construction related to the diffusion-convolution networks \cite{atwood_diffusion-convolutional_2016}, and the anisotropic heat flow used to construct patch operators in \cite{boscaini_learning_2016}. However, the focus here is on distributing orientations in $OM$ as opposed to constructing a density or defining patches on $M$.

  It is common practice to use a form of spatial pooling in convolutional networks. Average pooling is by construction convolution with a box kernel. With stride, it provides a coarsened version of the discretized output function similarly to max-pooling. The Gaussian weighting of the integral gives a similar effect when convolving the result of a convolution with an identity kernel: Letting $k_2(x)=1$ in \eqref{eq:two_conv}, we get $k_2\ast_{W_{T/2}} (k_1\ast_{W_{T/2}} f)(u)=\E[k_1(-(W_T-W_{T/2}))f(U_T^u)]$ where it can be seen that $f$ is evaluated at time $T$ of the Brownian motion whereas $k_1$ is evaluated at $W_T-W_{T/2}$ which has half the variance. This can be seen as a ``Gaussian stride'': $f$ is evaluated at points having twice the variance as the evaluation points of the kernel thus mimicking the way regular stride doubles the length scale on which the input function is evaluated. In the Euclidean situation, the result can be seen as exchanging the average filter in average pooling with a convolution of the output with a Gaussian kernel of larger width.

\subsection{Numerical implementation}
While the heat equation on manifolds is a nonlinear PDE, the heat kernel can be numerically computed efficiently on discretized surfaces. The vector heat method \cite{Sharp:2019:VHM} lifts this to transport in the tangent bundle. We expect these methods to be transferable to efficient numerical evaluation of the expectation in \eqref{eq:antidev_conv}. Though an actual implementation is left to future work, a sketch algorithm is listed in Algorithm~\ref{alg:conv}.
\begin{algorithm} \DontPrintSemicolon \SetAlgoLined
  1. Integrate the density $\rho:\R^d\times OM\to\R$ 
  $$\partial_t \rho_t^{(\mathbf 0,u)}=A^*\rho^{(\mathbf 0,u)}_t$$ from $t=0$ to $t=T$ with $\rho_0^{(\mathbf 0,u)}=\delta_{(\mathbf 0,u)}$ and $A$ being the generator for the $\R^d\times OM$-valued process $(W_t,U_t^u)$.
  \\
  2. Evaluate the integral $$\int_{\R^d\times OM} k(-\mathbf v)f(\tilde{u})\rho_T^{(0,u)}(\mathbf v,\tilde{u})d(\mathbf v,\tilde{u})$$
  \caption{Evaluation, convolution $k\ast_{W_T}f(u)$}
  \label{alg:conv}
\end{algorithm} 

On 2D surfaces where the fibers $\pi^{-1}(x)$ are isomorphic to the circle $\SO(1)$, $OM$ can be discretized by a triangulation of $M$ together with a division of $\SO(1)$ in a finite set of bins. The kernel $k$ will in practice also be discrete giving a natural discretization of $\R^d$. %The PDA flow in Algorithm~\ref{alg:conv} can then be computed as a sequence of matrix operations. 
Importantly, the integration in step 1. of Algorithm~\ref{alg:conv} can be precomputed prior to training and prediction as only the evaluation of the integral (step 2.) depends on $f$ and $k$. Numerical integration of the density $\rho$ is therefor not needed at training or prediction time. In addition, the integration needs only be precomputed once for each $x\in M$ since equivariance implies that computing it for one $u\in\pi^{-1}(x)$ makes the value available for all elements of the fiber.

%Various approximations can be applied to the computation of the heat flow, for example, the off-diagonal parts of the product $\R^d\times OM$ can be ignored giving the approximation $p_T^{(0,u)}\approx p_T^0\cdot p_T^u$ where $p_T^0$ is the standard $\R^d$ heat kernel, and $p_T^u$ is computed by integrating the $OM$ flow $\partial_t p_t^u=-\frac12\nabla_{OM}p^u_t$. We leave evaluation of such approximations to future work.

%Because the convolution \eqref{eq:antidev_conv} inherits the smoothness of the input function when $M$ is, e.g., analytic, it is possible to backpropagate through the operation, the actual operation depending on the chosen numerical scheme. Because $U_t^u$ is a horizontal flow, the computation will involve derivatives of the horizontal vector fields, and thus be related to the curvature of $M$.

Examples of numerical implementation of stochastic horizontal transport, development and Monte Carlo approximation of $p_t(\cdot; x)$ can be found in the Theano Geometry framework \cite{kuhnel_differential_2019}.

\section{Sampling means for manifold valued convolutional filters}
We now switch focus and consider the situation of a manifold valued convolutional filter, i.e. $k\ast f$ takes values in $M$. The complexity here lies in the fact that there is no direct way to enforce the value of an integral to take values in a manifold. This problem has been the focus of intensive interest in geometric statistics, the statistical analysis of data in geometric spaces: Fr\'echet defined in \cite{frechet_les_1948} a generalization of the Euclidean expected valued as the Fr\'echet mean (FM), and this and related notions of manifold means have been treated in numerous works. Relevant for the present context is the introduction of weights and the weighted Fr\'echet mean (wFM) which in \cites{pennec_riemannian_2006,chakraborty_manifoldnet_2020,chakraborty_deep_2019} is used to define a generalization of the Euclidean convolution that takes values in a manifold.

Because both the Fr\'echet mean and the weighted Fr\'echet mean are posed as optimization problems -- minimizers of the (weighted) variance, they are typically expensive to compute, which is a major obstacle in deep learning applications. This issue can be handled in spaces where geodesics have closed form solution using an inductive estimator \cites{chakraborty_manifoldnet_2020,chakraborty_deep_2019}. Here, we take a different view on the estimation problem and propose a method for \emph{sampling} from a distribution centered around weighted means, thus removing the need for optimization steps to find shortest geodesics and a minimizer of the expected variance. For this, we introduce the \emph{weighted diffusion mean}, a version of mean value that is defined from a likelihood principle in contrast to the non-parametric definition of the Fr\'echet mean. We develop a novel sampling scheme in the $n$-fold product manifold $M^k$ for $k$ points by conditioning a stochastic process to hit the diagonal of $M^k$, identify the distribution of the resulting random variable, and relate the introduced stochasticity to other stochastic neural networks models.

\subsection{Background}
Euclidean convolution can be written $k\ast f(\mathbf x)=\E[k(\mathbf x-\mathbf z)f(\mathbf z)]$ with expectation with respect to the Lebesgue measure. Assuming $\E[k]=1$, the result can equivalently be expressed as 
$k\ast f(\mathbf x)=\argmin_{\mathbf y\in\R}\E[k(\mathbf x-\mathbf z)\|\mathbf y-f(\mathbf z)\|^2]$ where $\|\cdot\|^2$ denotes the squared Euclidean norm by differentiating at optimal $\mathbf y$, see e.g. \cite{goh_nonparametric_2011}. While the expected value does not have a manifold equivalent, the Riemannian generalization $\argmin_{y\in\M}\E[k(z)d(y,f(z))^2]$ of the variational formulation has solutions which are denoted \emph{weighted Fr\'echet means}, see e.g. \cite{lim_weighted_2014} (local minimizers are denoted weighted Karcher means). The weighted Fr\'echet mean is in \cites{pennec_riemannian_2006,chakraborty_manifoldnet_2020,chakraborty_deep_2019} used to define a manifold valued convolution operator: For $x_1,\dots,x_n\in M$ and weights $w_1,\dots,w_n\in\R$, the generalized convolution is $\argmin_{y\in M}\sum_{i=1}^Nw_id(y,x_i)^2$. This can be formulated in a continuous setting for $f:M\to M$, $x\in \M$ as $k\ast f(x)=\argmin_{y\in M}\E[k(x,z)d(y,f(z))^2)]$ where $k:M\times M\to\R$ is the kernel satisfying $\E[k(x,\cdot)]=1$ for each $x$.\footnote{\cite{chakraborty_manifoldnet_2020,chakraborty_deep_2019} define the convolution for an $M$-valued random variable $X$ by $\argmin_{y\in M}\E[k(X)d(y,X)^2)]$} % and, for $k:\R^d\to\R,f:\R^d\to M$, $k\ast f(x)=\argmin_{y\in M}\E[k(-v)d(y,f(\Exp_x(v))^2)]$.

Generally, computing the weighted Fr\'echet mean is expensive requiring solution of nested optimization problems: Each computation $d(y,x_i)^2$ for a candidate $y$ includes, for each $x_i$, solving an optimization problem to find the squared length of a minimizing geodesic on $M$, and this computation needs to be iterated in each step of an iterative, gradient based optimization to find an optimal $y$. This is clearly not adequate for deep learning applications. \cites{chakraborty_manifoldnet_2020,chakraborty_deep_2019} propose to use an inductive estimator that computes an estimate of the wFM by computation of $n-1$ geodesics between the candidate point and the input $x_i$. This computation is efficient in the cases where geodesics can be computed efficiently, e.g. in closed form. In this case, it is possible to backpropagate through the wFM estimation, and thereby to use the convolution layer in a standard deep network setup.

The Fr\'echet mean as used in geometric statistics has a cousin in the \emph{diffusion mean} (DM \cite{hansen_diffusion_2021,hansen_diffusion_2021-1}, also denoted \emph{Brownian motion maximum likelihood mean}, see e.g. \cite{said_extrinsic_2012,sommer_anisotropic_2015,sommer_modelling_2017}). This definition uses that the Euclidean expected value has an equivalent definition as the maximally likely center point of a normal distribution fitted to data: If $p_\theta:\R^d\to\R$ denotes the density of a normal distribution $\mathcal{N}(\mathbf y,\sigma^2)$ with parameter $\theta=(\mathbf y,\sigma^2)$, then the maximizer of the log-likelihood $\bar{\theta}=\argmax_\theta\E[\log p_\theta(\mathbf X)]$ for an $\R^d$-valued random variable $\mathbf X$ has the expected value in the $\mathbf y$-variable, i.e. $\bar{\mathbf y}=\E[\mathbf X]$. Since the normal distribution can be generalized to manifolds with the Riemannian Brownian motion (this is one among other generalizations, see \cite{pennec_intrinsic_2006}), an equivalent Riemannian definition of mean value is $\argmax_{y\in\M}\E[\log p_T(X; y)]$ where, as earlier on, $p_T(\cdot; y)$ denotes the solution of the heat flow started at $y\in M$, i.e. the density of the Riemannian Brownian motion, and $T$ takes the role of the variance $\sigma^2$. The interest in the DM lies in the natural incorporation of the curvature of $M$ for data with large spread. In low dimensions, $p_T(\cdot; y)$ can be approximated directly using spectral methods, whereas in high dimensions, $p_T(\cdot; y)$ can be approximated by sampling the Brownian motion conditioned on hitting the data, see e.g. \cite{jensen_simulation_2021} and below.

\subsection{Weighted diffusion mean}
We here propose a scheme related to the use of the wFM for manifold-valued convolution, but we exchange the wFM with a \emph{weighted} diffusion mean (wDM): In Euclidean space, the weighted mean equals the maximally likely center point $\mathbf y$ of independent samples $x_1,\dots,x_n$ from $n$ normal distributions $\mathcal{N}(\mathbf y,T/w_1),\dots,\mathcal{N}(\mathbf y,T/w_n)$. Again taking the Brownian motion as the manifold equivalent of the Euclidean normal distribution with density $p_{T/w_i}(\cdot; y)$, we here define the weighted diffusion mean wDM as $\argmax_y\sum_{i=1}^n\log p_{T/w_i}(x_i; y)$ (discrete version) and $\argmax_y\E[\log p_{T/k(z)}(z; y)]$ (continuous version). As we will see below, the probabilistic nature of the mean allows sampling from a distribution centered at the mean, thereby removing the need for optimization to find geodesics.

Similarly to the use of the wFM for convolution, the manifold-valued convolution using the wDM is here 
\begin{equation*}
  k\ast f(x)=\argmax_{y\in M}\E[\log p_{T/k(x,z)}(f(z);y)]
\end{equation*}
with $k:M\times M\to\R$ and $f:M\to M$.% and, for $k:\R^d\to\R,f:\R^d\to M$, $k\ast f(x)=\argmax_{y\in M}\E[\log p_{k(-v)^2}d(y,f(\Exp_x(v))^2)]$.

\subsection{Bridge sampling for likelihood approximation}
\label{sec:bridge_sampling}
We now switch the roles of $y$ and the data $x_i$: In the Euclidean setting, we consider the probability of observing $\mathbf y$ in each of the $n$ distributions $\mathcal{N}(\mathbf x_i,T/w_i)$ simultaneously. The distribution of $\mathbf y$ is then $\mathcal{N}(\tfrac{\sum_{i=1}^nw_i\mathbf x_i}{\sum_{i=1}^nw_i},\tfrac{T}{\sum_{i=1}^nw_i})$, i.e. again a normal distribution however centered at the weighted mean. On manifolds, it is computationally difficult to compute the wDM directly similarly to the wFM, but it turns out we can sample from a manifold equivalent of the distribution of $\mathbf y$. We achieve this by sampling a conditioned Brownian motion in the $n$-fold product manifold $M^n$. Below, we first discuss sampling the conditioned distribution in the one sample situation ($n=1$, $w_1=1$) before moving on to the weighted case.

Let $y,v\in M$, and let $X_t^y$ denote the Brownian motion starting at $y$. The time $t=T$ conditioned process $X_t^x|X_T=v$, a Brownian bridge, has the property of hitting the target value $v$ at time $T$ a.s. Sampling of the conditioned process is often of interest because it can be used to approximate the heat kernel $p_T(v; y)$. That is, if we can sample a process that approximates $X_t^y|X_T=v$, we can approximate the heat kernel even in high dimensions were direct solution of the heat PDE is not applicable. In general, for an It\^o stochastic differential equation (SDE)
\begin{equation*}
  dx_t=
  b(x_t)dt
  +
  a(x_t)
  dW(t)
  \ ,
\end{equation*}
the corresponding bridge process hitting $v$ at time $T$ satisfies the SDE
\begin{equation}
  dx_t=
  b(x_t)dt
  +
  a(x_t)
  a(x_t)^T
  \nabla\log p_{T-t}(v; x_t)dt
  +
  a(x_t)
  dW(t)
  \ .
  \label{eq:bridge}
\end{equation}
Note the addition of the extra drift term that includes the gradient of the $\log$-density $\nabla\log p_{T-t}(v; x_t)$. In the present case, writing the Brownian motion in local coordinates, the drift of the SDE is $b(x)=-\frac12g(x)^{kl}\Gamma(x)\indices{_{kl}}$ and the diffusion coefficient $a(x)=\sqrt{g^{-1}(x)}$, i.e. the drift is a contraction over the Christoffel symbols, and the diffusion coefficient is a square root of the inverse of the Riemannian metric $g$. 

The bridge SDE \eqref{eq:bridge} is however not useful for computational purposes since we cannot expect to be able to compute $\nabla\log p_{T-t}(v; x_t)$ at each time step. Instead, to handle this fact, Delyon and Hu \cite{delyon_simulation_2006} introduced a guided approximation of the bridge process. This consist of an SDE
\begin{equation}
  dx_t=
  b(x_t)dt
  +
  \tilde{b}(x_t)dt
  +
  a(x_t)
  dW(t)
  \label{eq:approx_bridge}
\end{equation}
where the term $a(x_t)a(x_t)^T\nabla\log p_{T-t}(v; x_t)$ in \eqref{eq:bridge} is exchanged with a computationally feasible alternative, either $\tilde{b}(x)=\frac{v-x}{T-t}dt$, or, alternatively, $\tilde{b}(x)=a(x_t)a(x_t)^T\nabla\log \tilde{p}_{T-t}(v; x)$ where $\tilde{p}$ is a density of a simpler process with closed form transition density \cite{schauer_guided_2017}. Here, we follow the former approach as used in \cite{jensen_simulation_2021} to sample the Brownian motion. The transition density can then be written in the form 
\begin{equation*}
  p(T,v; y)
  =
  \sqrt{\frac{|(a(x)^{-1})^Ta(x)|}{(2\pi T)^d}}
  e^{\frac{-\|a(y)^{-1}(y-v)\|^2}{2T}}
  \E[\varphi(x_t)]
\end{equation*}
where $\varphi$ denotes a correction factor between the law of the true bridge and the law of the guided proposal process. Note that $p(T,v; y)$ is then a product of a term which is the density of the standard $\R^d$ normal distribution, and the expected correction $\E[\varphi(x_t)]$. The latter of these terms, which encodes the deviation from the Euclidean situation, can be approximated numerically by Monte Carlo sampling.
%with $A(x)=(a(x)^{-1})^Ta(x)$.
The guided proposal scheme is illustrated in Figure~\ref{fig:guided} (left). 
\begin{figure}[t]
  \centering
  \def\svgwidth{.48\columnwidth}
  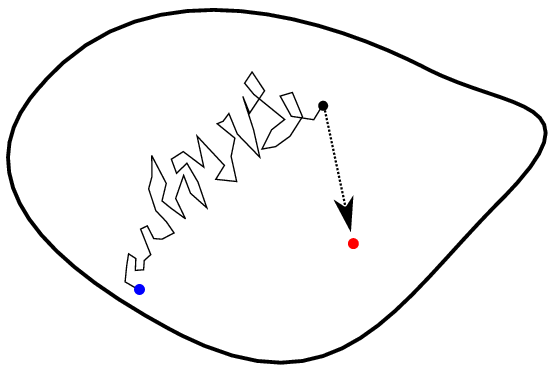 
  \def\svgwidth{.48\columnwidth}
  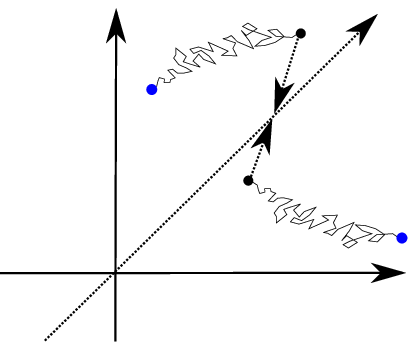 
  \caption{
    (left) On $M$, the guided proposal scheme \eqref{eq:approx_bridge} forces the process to hit the target $v$ at time $T$ a.s. by adding the drift $\tilde{b}$ that is typically the difference $v-x(t)$ (in coordinates, dotted arrow in figure) scaled by the inverse remaining time $(T-t)^{-1}$. The difference between the law of the conditioned and the guided process is proportional to the factor $\varphi(x_t)$. (right) On the product $M^2$ (or $M^n$), we can apply the guidance scheme to force independent Brownian motions to hit each other at time $T$. This is done by adding a drift $\tilde{b}$ that forces the processes towards the diagonal by adding the differences to the weighted arithmetic mean (in coordinates) of the processes. The result is the $\diag(M^2)$ valued random variable $v$.
        }
  \label{fig:guided}
\end{figure}

\subsection{Sampling the wDM}
However, the goal here is not to estimate $p_T(v; y)$ but to sample a distribution approximating the wDM. We use the ideas above to turn the problem around in the following way: For observations $x_1,\dots,x_n$, let $M^n$ denote the product manifold with the product Riemannian metric. We then start a process $x_t=(x_{1,t},\dots,x_{n,t})$ at the point $(x_1,\dots,x_n)\in M^n$ and condition it on having equal components at time $t=T$, i.e. $x_{1,T}^n=\cdots=x_{n,T}^n$. That is, the process must hit the diagonal of the product space $M^n$ similarly to the simultaneous observation of $y$ in $n$ normal distributions in the Euclidean situation described at the start of section~\ref{sec:bridge_sampling}. The processes $x_{i,t}$ are independent Brownian motions with variance $T/w_i$. This gives the conditioned process
\begin{equation}
  (x_{1,t},\dots,x_{n,t})|x_{1,T}=\cdots=x_{n,T}
  \ .
  \label{eq:product_bridge}
\end{equation}
The conditioned process is analogous to the Brownian bridge except that we condition on a subspace in $M^n$ instead of a point in $M$. In essence, the process runs backwards from the observations to reach a point $v=x_{1,T}^n=\cdots=x_{n,T}$ in $M$. Due to the symmetry $p_{T/w_i}(y; x)=p_{T/w_i}(x; y)$ of the Brownian motion, and the independence of the individual Brownian motions $x_{i,t}$ on $M^n$, we have $p_{T/w_i}(v; (x_1,\dots,x_n))=\prod_{i=1}^np_{T/w_i}(v; x_i)=\prod_{i=1}^np_{T/w_i}(x_i; v)$, i.e. the probability of observing $v$ at the diagonal equals the probability of observing $x_1,\dots,x_n$ on $M$ regarding $v$ as a parameter.

Similarly to \eqref{eq:bridge}, the conditioned process has an SDE that depends on the (intractable) log-transition density. However, we can again construct a guided process \eqref{eq:approx_bridge} on the product manifold $M^n$ using a drift which in a coordinate chart reads $\tilde{b}(x_{1,t},\dots,x_{n,t})=((\mu(t)-x_{1,t})/w_1,\dots,(\mu(t)-x_{n,t})/w_n)/(T-t)$ with $\mu(t)=\tfrac{\sum_{i=1}^nw_ix_i(t)}{\sum_{i=1}^nw_i}$. The scheme is illustrated in Figure~\ref{fig:guided} (right). We let $\varphi$ denote the correction factor as above.

We then obtain the following result.
\begin{theorem}
  Let $x_t=(x_{1,t},\dots,x_{t,n})$ consist of $n$ independent Brownian motions on $M$ with variance $T/w_i$, and let $\tilde{x}_t$ be the process with additional added drift $\tilde{b}$. Let $\mathbb P$ be the law of $\tilde{x}_t$; $\mathbb P^*$ the law of the conditioned process \eqref{eq:product_bridge}; and $\varphi$ the correction factor of the guided process as in \eqref{eq:approx_bridge}. Let $v(x_1,w_1,\dots,x_n,w_n)$ be the random variable $\tilde{x}_{1,T}$ with law $\tfrac{\varphi}{\E[\varphi]} \mathbb P$. Then $v$ has a density $p_v(y)\propto \prod_{i=1}^np_{T/w_i}(x_i; y)$ and $v=\tilde{x}_{i,T}$ for all $i$ a.s.%; and the set of modes of $p_v$ equal the set of the maximizers for the likelihood $L(y; x_1,\dots,x_n)=\prod_{i=1}^np_{T/w_i}(x_i; y)$.
  \label{thm:product_processes}
\end{theorem}
\begin{proof}[Proof outline]
  That the construction of \cite{delyon_simulation_2006} extends to conditioning on subspaces is shown in \cite{thompson_brownian_2016}.
  The distribution of $v$ with respect to the probability measure $\frac{\varphi}{\E[\varphi]} \mathbb P$ equals the distribution of $x_{1,T}^n$ with respect to $\mathbb P^*$. Because the processes $x_{i,t}$ are independent, $p_v(y)\propto \prod_{i=1}^np_{T/w_i}(x_i; y)$. The differences $\mu(T)-x_{i,T}$ are $0$ similarly to the case of \cite{delyon_simulation_2006,thompson_brownian_2016} showing that $x_{i,T}=x_{j,T}$ for all $i,j=1,\dots,n$.
   
\end{proof}
As a consequence, the Euclidean situation with normally distributed $y$ persists in the manifold situation: The random variable $v$ arise from observing the same value $v$ in $n$ independent Brownian motions. The weighting appears as a scaling of the variances of the individual processes. The proof as outlined here omits details. % including when the process shifts between coordinate charts. 
A rigorous argument will be presented in a future paper. 

%If the is non-trivial to show that the likelihood ratio $\tfrac{\varphi}{\E[\varphi]}$ is bounded which is needed for the rejection sampler below. We here assume such a bound $C$.
%\begin{remark}
%  In the Euclidean situation with unique wDM $y$, $v$ is normally distribution as mentioned above and $v$ can be regarded an estimator of $y$ with variance $T/\sum_{i=1}^nw_i$. In the nonlinear situation, Theorem~\ref{thm:product_processes} implies that the distribution of $v$ has mode at $y$. In the asymptotic limit $n\to\infty$, the density $p_v$ will approach the density of a rescaled normal distribution in $T_yM$, however, the curvature implies that $p_v$ is only approximately normal for finite $n$.
%\end{remark}
We can now sample an approximation of the wDM by accounting for $\varphi$ with the following sampling importance resampler (SIR, Algorithm \ref{alg:inference}).
\begin{algorithm} \DontPrintSemicolon \SetAlgoLined
  sample $J$ paths from the guided process $\tilde{x}_t^j$\\
   compute correction factors $\varphi^1,\dots,\varphi^J$\\
   sample $j$ from $1,\dots,J$ with probability $\{\varphi^j/\sum_{j=1}^J\varphi^j\}$\\
   return $v=\tilde{x}_{j,T}$\\
  \caption{wDM estimation by SIR}
  \label{alg:inference}
\end{algorithm}
The algorithm as listed is written in coordinates assuming relevant charts. Alternatively, if $M$ is embedded as a subset of  $\R^k$, $\R^k$ coordinates can be used. The algorithm requires the computation of the Christoffel symbols in the integration of $x_t$ as is required for the numerical integration of geodesics. However, importantly, it removes any need for nested iterative optimization as is used for the wFM in cases where geodesics do not have closed form solutions. Changes to the weights $w_i$ affect the sample paths $x_t^j$ and corrections $\varphi^j$. The coupling between $w_i$ and $\varphi$ is however only through interaction between the guidance term and the Christoffel symbols. In practice, this can be ignored allowing backpropagation through the algorithm.% A bound on $\E[\varphi]C$ can be estimated once by sampling.

Note that because the wDM approaches the wFM when $T$ is small, samples from Algorithm~\ref{alg:inference} will in this case approach the wFM. However, small $T$ will affect the probability of the samples ($\varphi$ will tend to zero) because any deviance of the guided process from a geodesic will be less likely. Nonzero $T$ can thus be seen as a way to get computational efficiency at the cost of variance in the estimator.
Conversely, stochasticity can be reduced by lowering the evaluation time $T$ of the Brownian motion. This will reduce the variance of $v$, and result in the wDM approaching the wFM. However, this will require larger $J$ in algorithm~\ref{alg:inference} and thus increase in computational cost.

\subsection{Stochastic NN outputs}
While the stochasticity of the wDM estimator adds randomness to an otherwise deterministic setup, the added stochasticity is rather natural. For example, the deep Gaussian process model employed in \cite{gal_dropout_2016} when using dropout for uncertainty quantification uses the data conditional distribution $\mathbf y|\mathbf x,w=\mathcal{N}(\hat{y}(\mathbf x,w),\tau^{-1})$ for the output $\mathbf y$ given the input $\mathbf x$, weights $w$, deterministic neural network output $\hat{\mathbf y}(\mathbf x,w)$, and precision parameter $\tau$. Here, we get the same distribution of $v$ with $\tau^{-1}=T/\sum_{i=1}^nw_i$. $\sum_{i=1}^nw_i/T$ can therefore be regarded a precision parameter for the model. Comparing to Euclidean networks, because the stochasticity is built into the convolution operator, stochasticity is here added before application of a nonlinearity, while the output in \cite{gal_dropout_2016} is normally distributed after application of nonlinearities in $\hat{\mathbf y}(\mathbf x,\mathbf w)$. Similarly to the Monte Carlo sampling of the moments of the predictive distribution in \cite{gal_dropout_2016}, Algorithm~\ref{alg:inference} can be used to estimate the moments of the output by using the correction factors $\varphi^j$ as importance sampling weights.

\section{Conclusion and outlook}
%We have outlined the fiber bundle geometry that allows the use of parallel transport in geometric deep learning to be viewed as integral curves of horizontal flows, and we used this to explicitly link the lack of associativity and commutativity of convolutions using parallel transport to the curvature of the manifold. Subsequently, we used stochastic flows in the orthonormal frame bundle to distribute orientations globally over $M$, and we applied development and anti-development to define a new convolution operator that naturally includes orientations in its definition trough the use of frame bundle flows. The construction does not rely on minimizing geodesics and therefore removes the orientation ambiguity when using kernel with large support.

%We can exchange the wFM as used for defining manifold valued convolutional layers with the wDM estimator $v(x_1,w_1,\dots,x_n,w_n)$. By making the output stochastic, the need for iterative optimization to find the wFM is removed, and the computational effort is therefore lower. This is in particular important for manifolds without closed form solutions for geodesics.

The paper concerned the application of fiber bundle geometry and methods from stochastic analysis on manifolds in geometric deep learning in two cases: convolution with manifold domain, and convolution with manifold target. We showed how horizontal flows in the frame bundle provides a direct way of quantifying the role of curvature in the non-commutativity of the convolution when using parallel transport along minimizing geodesics. We then used this insight, the stochastic flows in the Eells-Elworthy-Mallivin construction of Brownian motion, and stochastic development, to define a new convolution operator that in a natural way constructs a distribution of orientations globally on the manifold. The anti-development of the Brownian motion allows kernels on $\R^d$ to be applied in a seamless way. In addition, the distribution of orientations in the frame bundle allows evaluation of fully connected layers that incorporates global information over the manifold without pooling over orientations.

In the second part of the paper, we showed how the weighted diffusion mean can be used to define a convolution that takes values in a manifold. By conditioning a stochastic process in the $n$-fold product space $M^n$, we obtain a method for sampling from a distribution that centers at the wDM. This removes the need for nested iterative optimization for computing the weighted Fr\'echet in cases where geodesics do not have closed form solution, and thereby allows the convolution operator to be applied on a much more general class of manifolds.

We briefly discussed computational aspects and algorithms, however, we here focused on introducing the theoretical constructions and foundations for applying nonlinear stochastic methods in geometric deep learning. We hope that this will inspire further developments in the field, both in its theoretical foundation and in development of efficient algorithms.

\ifCLASSOPTIONcompsoc
  \section*{Acknowledgments}
\else
  \section*{Acknowledgment}
\fi

The work presented in this article is supported by the CSGB Centre for Stochastic Geometry and Advanced Bioimaging funded by a grant from the Villum foundation, 
the Villum Foundation grant 00022924, and the Novo Nordisk Foundation grant NNF18OC0052000.

\ifCLASSOPTIONcaptionsoff
  \newpage
\fi



\bibliographystyle{IEEEtran}
\bibliography{library}
%\begin{thebibliography}{1}
%\end{thebibliography}

\end{document}

%% file: figures/backwards.eps_tex
%% Accompanies image file 'backwards.eps' (pdf, eps, ps)
\begingroup%
  \makeatletter%
  \providecommand\color[2][]{%
    \errmessage{(Inkscape) Color is used for the text in Inkscape, but the package 'color.sty' is not loaded}%
    \renewcommand\color[2][]{}%
  }%
  \providecommand\transparent[1]{%
    \errmessage{(Inkscape) Transparency is used (non-zero) for the text in Inkscape, but the package 'transparent.sty' is not loaded}%
    \renewcommand\transparent[1]{}%
  }%
  \providecommand\rotatebox[2]{#2}%
  \newcommand*\fsize{\dimexpr\f@size pt\relax}%
  \newcommand*\lineheight[1]{\fontsize{\fsize}{#1\fsize}\selectfont}%
  \ifx\svgwidth\undefined%
    \setlength{\unitlength}{363.6549005bp}%
    \ifx\svgscale\undefined%
      \relax%
    \else%
      \setlength{\unitlength}{\unitlength * \real{\svgscale}}%
    \fi%
  \else%
    \setlength{\unitlength}{\svgwidth}%
  \fi%
  \global\let\svgwidth\undefined%
  \global\let\svgscale\undefined%
  \makeatother%
  \begin{picture}(1,0.65833335)%
    \lineheight{1}%
    \setlength\tabcolsep{0pt}%
    \put(0,0){\includegraphics[width=\unitlength]{backwards.eps}}%
    \put(0.78073204,0.45556851){\color[rgb]{0,0,0}\makebox(0,0)[lt]{\lineheight{1.25}\smash{\begin{tabular}[t]{l}$\pi|_{OM}^{-1}(x)$\end{tabular}}}}%
    \put(0.76953636,0.28659954){\color[rgb]{0,0,0}\makebox(0,0)[lt]{\lineheight{1.25}\smash{\begin{tabular}[t]{l}$u$\end{tabular}}}}%
    \put(0.3707577,0.19305521){\color[rgb]{0,0,0}\makebox(0,0)[lt]{\lineheight{1.25}\smash{\begin{tabular}[t]{l}$\R^d$\end{tabular}}}}%
    \put(0.42633915,0.07730532){\color[rgb]{0,0,0}\makebox(0,0)[lt]{\lineheight{1.25}\smash{\begin{tabular}[t]{l}$\mathbf v_2$\end{tabular}}}}%
    \put(0.525465,0.15476038){\color[rgb]{0,0,0}\makebox(0,0)[lt]{\lineheight{1.25}\smash{\begin{tabular}[t]{l}$\mathbf v_1$\end{tabular}}}}%
    \put(0.40942761,0.38485797){\color[rgb]{0,0,0}\makebox(0,0)[lt]{\lineheight{1.25}\smash{\begin{tabular}[t]{l}$u\mathbf v_2$\end{tabular}}}}%
    \put(0.09292839,0.36873747){\color[rgb]{0,0,0}\makebox(0,0)[lt]{\lineheight{1.25}\smash{\begin{tabular}[t]{l}$Pu\mathbf v_1$\end{tabular}}}}%
    \put(0.30225657,0.54395697){\color[rgb]{0,0,0}\makebox(0,0)[lt]{\lineheight{1.25}\smash{\begin{tabular}[t]{l}$\R$\end{tabular}}}}%
    \put(0.19015087,0.55979323){\color[rgb]{0,0,0}\makebox(0,0)[lt]{\lineheight{1.25}\smash{\begin{tabular}[t]{l}$f$\end{tabular}}}}%
    \put(0.61791744,0.55099534){\color[rgb]{0,0,0}\makebox(0,0)[lt]{\lineheight{1.25}\smash{\begin{tabular}[t]{l}$\R$\end{tabular}}}}%
    \put(0.41506625,0.56683159){\color[rgb]{0,0,0}\makebox(0,0)[lt]{\lineheight{1.25}\smash{\begin{tabular}[t]{l}$k_1\ast f$\end{tabular}}}}%
    \put(0.93935643,0.55320505){\color[rgb]{0,0,0}\makebox(0,0)[lt]{\lineheight{1.25}\smash{\begin{tabular}[t]{l}$\R$\end{tabular}}}}%
    \put(0.71588127,0.5690413){\color[rgb]{0,0,0}\makebox(0,0)[lt]{\lineheight{1.25}\smash{\begin{tabular}[t]{l}$k_2\ast f_1$\end{tabular}}}}%
  \end{picture}%
\endgroup%

%% file: figures/development.eps_tex
%% Accompanies image file 'development.eps' (pdf, eps, ps)
\begingroup%
  \makeatletter%
  \providecommand\color[2][]{%
    \errmessage{(Inkscape) Color is used for the text in Inkscape, but the package 'color.sty' is not loaded}%
    \renewcommand\color[2][]{}%
  }%
  \providecommand\transparent[1]{%
    \errmessage{(Inkscape) Transparency is used (non-zero) for the text in Inkscape, but the package 'transparent.sty' is not loaded}%
    \renewcommand\transparent[1]{}%
  }%
  \providecommand\rotatebox[2]{#2}%
  \newcommand*\fsize{\dimexpr\f@size pt\relax}%
  \newcommand*\lineheight[1]{\fontsize{\fsize}{#1\fsize}\selectfont}%
  \ifx\svgwidth\undefined%
    \setlength{\unitlength}{297.63779528bp}%
    \ifx\svgscale\undefined%
      \relax%
    \else%
      \setlength{\unitlength}{\unitlength * \real{\svgscale}}%
    \fi%
  \else%
    \setlength{\unitlength}{\svgwidth}%
  \fi%
  \global\let\svgwidth\undefined%
  \global\let\svgscale\undefined%
  \makeatother%
  \begin{picture}(1,0.5154262)%
    \lineheight{1}%
    \setlength\tabcolsep{0pt}%
    \put(0,0){\includegraphics[width=\unitlength]{development.eps}}%
    \put(0.72200022,0.44525997){\color[rgb]{0,0,0}\makebox(0,0)[lt]{\lineheight{1.25}\smash{\begin{tabular}[t]{l}$x$\end{tabular}}}}%
    \put(0.02691115,0.35147175){\color[rgb]{0,0,0}\makebox(0,0)[lt]{\lineheight{1.25}\smash{\begin{tabular}[t]{l}$\R^d$\end{tabular}}}}%
    \put(0.77333689,0.20547013){\color[rgb]{0,0,0}\makebox(0,0)[lt]{\lineheight{1.25}\smash{\begin{tabular}[t]{l}$U_t$\end{tabular}}}}%
    \put(0.25830298,0.1047877){\color[rgb]{0,0,0}\makebox(0,0)[lt]{\lineheight{1.25}\smash{\begin{tabular}[t]{l}$W_t$\end{tabular}}}}%
  \end{picture}%
\endgroup%

%% file: figures/horz_conv.eps_tex
%% Accompanies image file 'horz_conv.eps' (pdf, eps, ps)
\begingroup%
  \makeatletter%
  \providecommand\color[2][]{%
    \errmessage{(Inkscape) Color is used for the text in Inkscape, but the package 'color.sty' is not loaded}%
    \renewcommand\color[2][]{}%
  }%
  \providecommand\transparent[1]{%
    \errmessage{(Inkscape) Transparency is used (non-zero) for the text in Inkscape, but the package 'transparent.sty' is not loaded}%
    \renewcommand\transparent[1]{}%
  }%
  \providecommand\rotatebox[2]{#2}%
  \newcommand*\fsize{\dimexpr\f@size pt\relax}%
  \newcommand*\lineheight[1]{\fontsize{\fsize}{#1\fsize}\selectfont}%
  \ifx\svgwidth\undefined%
    \setlength{\unitlength}{112.74603536bp}%
    \ifx\svgscale\undefined%
      \relax%
    \else%
      \setlength{\unitlength}{\unitlength * \real{\svgscale}}%
    \fi%
  \else%
    \setlength{\unitlength}{\svgwidth}%
  \fi%
  \global\let\svgwidth\undefined%
  \global\let\svgscale\undefined%
  \makeatother%
  \begin{picture}(1,1.18166768)%
    \lineheight{1}%
    \setlength\tabcolsep{0pt}%
    \put(0,0){\includegraphics[width=\unitlength]{horz_conv.eps}}%
    \put(0.36240495,1.13022063){\color[rgb]{0,0,0}\makebox(0,0)[lt]{\lineheight{1.25}\smash{\begin{tabular}[t]{l}$\pi|_{OM}^{-1}(x)$\end{tabular}}}}%
    \put(0.38236104,0.54340925){\color[rgb]{0,0,0}\makebox(0,0)[lt]{\lineheight{1.25}\smash{\begin{tabular}[t]{l}$u$\end{tabular}}}}%
  \end{picture}%
\endgroup%

%% file: figures/antidev_conv.eps_tex
%% Accompanies image file 'antidev_conv.eps' (pdf, eps, ps)
\begingroup%
  \makeatletter%
  \providecommand\color[2][]{%
    \errmessage{(Inkscape) Color is used for the text in Inkscape, but the package 'color.sty' is not loaded}%
    \renewcommand\color[2][]{}%
  }%
  \providecommand\transparent[1]{%
    \errmessage{(Inkscape) Transparency is used (non-zero) for the text in Inkscape, but the package 'transparent.sty' is not loaded}%
    \renewcommand\transparent[1]{}%
  }%
  \providecommand\rotatebox[2]{#2}%
  \newcommand*\fsize{\dimexpr\f@size pt\relax}%
  \newcommand*\lineheight[1]{\fontsize{\fsize}{#1\fsize}\selectfont}%
  \ifx\svgwidth\undefined%
    \setlength{\unitlength}{363.6549005bp}%
    \ifx\svgscale\undefined%
      \relax%
    \else%
      \setlength{\unitlength}{\unitlength * \real{\svgscale}}%
    \fi%
  \else%
    \setlength{\unitlength}{\svgwidth}%
  \fi%
  \global\let\svgwidth\undefined%
  \global\let\svgscale\undefined%
  \makeatother%
  \begin{picture}(1,0.65833335)%
    \lineheight{1}%
    \setlength\tabcolsep{0pt}%
    \put(0,0){\includegraphics[width=\unitlength]{antidev_conv.eps}}%
    \put(0.78073204,0.45556851){\color[rgb]{0,0,0}\makebox(0,0)[lt]{\lineheight{1.25}\smash{\begin{tabular}[t]{l}$\pi|_{OM}^{-1}(x)$\end{tabular}}}}%
    \put(0.76953636,0.28659954){\color[rgb]{0,0,0}\makebox(0,0)[lt]{\lineheight{1.25}\smash{\begin{tabular}[t]{l}$u$\end{tabular}}}}%
    \put(0.3707577,0.19305521){\color[rgb]{0,0,0}\makebox(0,0)[lt]{\lineheight{1.25}\smash{\begin{tabular}[t]{l}$\R^d$\end{tabular}}}}%
    \put(0.30225657,0.54395697){\color[rgb]{0,0,0}\makebox(0,0)[lt]{\lineheight{1.25}\smash{\begin{tabular}[t]{l}$\R$\end{tabular}}}}%
    \put(0.19015087,0.55979323){\color[rgb]{0,0,0}\makebox(0,0)[lt]{\lineheight{1.25}\smash{\begin{tabular}[t]{l}$f$\end{tabular}}}}%
    \put(0.61791744,0.55099534){\color[rgb]{0,0,0}\makebox(0,0)[lt]{\lineheight{1.25}\smash{\begin{tabular}[t]{l}$\R$\end{tabular}}}}%
    \put(0.35881624,0.56674254){\color[rgb]{0,0,0}\makebox(0,0)[lt]{\lineheight{1.25}\smash{\begin{tabular}[t]{l}$k_1\ast_{W_{T/2}} f$\end{tabular}}}}%
    \put(0.93935643,0.55320505){\color[rgb]{0,0,0}\makebox(0,0)[lt]{\lineheight{1.25}\smash{\begin{tabular}[t]{l}$\R$\end{tabular}}}}%
    \put(0.67875813,0.57316608){\color[rgb]{0,0,0}\makebox(0,0)[lt]{\lineheight{1.25}\smash{\begin{tabular}[t]{l}$k_2\ast_{W_{T/2}} f_1$\end{tabular}}}}%
    \put(0.73757315,0.11243288){\color[rgb]{0,0,0}\makebox(0,0)[lt]{\lineheight{1.25}\smash{\begin{tabular}[t]{l}$\R$\end{tabular}}}}%
    \put(0.58370393,0.1610465){\color[rgb]{0,0,0}\makebox(0,0)[lt]{\lineheight{1.25}\smash{\begin{tabular}[t]{l}$k_1,k_2$\end{tabular}}}}%
    \put(0.53238067,0.35520148){\color[rgb]{0,0,0}\makebox(0,0)[lt]{\lineheight{1.25}\smash{\begin{tabular}[t]{l}$U_{T/2}^u$\end{tabular}}}}%
    \put(0.12270278,0.23568083){\color[rgb]{0,0,0}\makebox(0,0)[lt]{\lineheight{1.25}\smash{\begin{tabular}[t]{l}$U_T^{U_{T/2}^u}$\end{tabular}}}}%
    \put(0.47026687,0.15886591){\color[rgb]{0,0,0}\makebox(0,0)[lt]{\lineheight{1.25}\smash{\begin{tabular}[t]{l}$W_t$\end{tabular}}}}%
  \end{picture}%
\endgroup%

%% file: figures/guided_bridges.eps_tex
%% Accompanies image file 'guided_bridges.eps' (pdf, eps, ps)
\begingroup%
  \makeatletter%
  \providecommand\color[2][]{%
    \errmessage{(Inkscape) Color is used for the text in Inkscape, but the package 'color.sty' is not loaded}%
    \renewcommand\color[2][]{}%
  }%
  \providecommand\transparent[1]{%
    \errmessage{(Inkscape) Transparency is used (non-zero) for the text in Inkscape, but the package 'transparent.sty' is not loaded}%
    \renewcommand\transparent[1]{}%
  }%
  \providecommand\rotatebox[2]{#2}%
  \newcommand*\fsize{\dimexpr\f@size pt\relax}%
  \newcommand*\lineheight[1]{\fontsize{\fsize}{#1\fsize}\selectfont}%
  \ifx\svgwidth\undefined%
    \setlength{\unitlength}{164.40364207bp}%
    \ifx\svgscale\undefined%
      \relax%
    \else%
      \setlength{\unitlength}{\unitlength * \real{\svgscale}}%
    \fi%
  \else%
    \setlength{\unitlength}{\svgwidth}%
  \fi%
  \global\let\svgwidth\undefined%
  \global\let\svgscale\undefined%
  \makeatother%
  \begin{picture}(1,0.74435127)%
    \lineheight{1}%
    \setlength\tabcolsep{0pt}%
    \put(0,0){\includegraphics[width=\unitlength]{guided_bridges.eps}}%
    \put(0.14321685,0.61938974){\color[rgb]{0,0,0}\makebox(0,0)[lt]{\lineheight{1.25}\smash{\begin{tabular}[t]{l}$M$\end{tabular}}}}%
    \put(0.59500582,0.36749546){\color[rgb]{0,0,0}\makebox(0,0)[lt]{\lineheight{1.25}\smash{\begin{tabular}[t]{l}$v-x(t)$\end{tabular}}}}%
    \put(0.65426963,0.20587208){\color[rgb]{0,0,0}\makebox(0,0)[lt]{\lineheight{1.25}\smash{\begin{tabular}[t]{l}$v$\end{tabular}}}}%
    \put(0.25422543,0.45672096){\color[rgb]{0,0,0}\makebox(0,0)[lt]{\lineheight{1.25}\smash{\begin{tabular}[t]{l}$x(t)$\end{tabular}}}}%
  \end{picture}%
\endgroup%

%% file: figures/guided_bridges2.eps_tex
%% Accompanies image file 'guided_bridges2.eps' (pdf, eps, ps)
\begingroup%
  \makeatletter%
  \providecommand\color[2][]{%
    \errmessage{(Inkscape) Color is used for the text in Inkscape, but the package 'color.sty' is not loaded}%
    \renewcommand\color[2][]{}%
  }%
  \providecommand\transparent[1]{%
    \errmessage{(Inkscape) Transparency is used (non-zero) for the text in Inkscape, but the package 'transparent.sty' is not loaded}%
    \renewcommand\transparent[1]{}%
  }%
  \providecommand\rotatebox[2]{#2}%
  \newcommand*\fsize{\dimexpr\f@size pt\relax}%
  \newcommand*\lineheight[1]{\fontsize{\fsize}{#1\fsize}\selectfont}%
  \ifx\svgwidth\undefined%
    \setlength{\unitlength}{164.40364207bp}%
    \ifx\svgscale\undefined%
      \relax%
    \else%
      \setlength{\unitlength}{\unitlength * \real{\svgscale}}%
    \fi%
  \else%
    \setlength{\unitlength}{\svgwidth}%
  \fi%
  \global\let\svgwidth\undefined%
  \global\let\svgscale\undefined%
  \makeatother%
  \begin{picture}(1,0.74435127)%
    \lineheight{1}%
    \setlength\tabcolsep{0pt}%
    \put(0,0){\includegraphics[width=\unitlength]{guided_bridges2.eps}}%
    \put(0.18883632,0.61026585){\color[rgb]{0,0,0}\makebox(0,0)[lt]{\lineheight{1.25}\smash{\begin{tabular}[t]{l}$M$\end{tabular}}}}%
    \put(0.701143,0.01414415){\color[rgb]{0,0,0}\makebox(0,0)[lt]{\lineheight{1.25}\smash{\begin{tabular}[t]{l}$M$\end{tabular}}}}%
    \put(0.61398172,0.46054027){\color[rgb]{0,0,0}\makebox(0,0)[lt]{\lineheight{1.25}\smash{\begin{tabular}[t]{l}$\mathrm{diag}(M^2)$\end{tabular}}}}%
  \end{picture}%
\endgroup%